\documentclass[letterpaper, 10 pt, journal, twoside]{IEEEtran}
\usepackage{graphicx}
\usepackage{multicol}
\usepackage{subfig}
\usepackage{booktabs}

\usepackage{array}
\usepackage{amsmath,amssymb}
\usepackage{amsthm}
\usepackage[utf8]{inputenc}
\usepackage[english]{babel}

\usepackage{multirow}
\newcommand{\tabincell}[2]{\begin{tabular}{@{}#1@{}}#2\end{tabular}}
\usepackage[font=small]{caption}
\usepackage{amsmath}
\usepackage{amssymb}
\usepackage{float}

\DeclareMathOperator*{\argmin}{argmin}
\usepackage{lettrine}
\usepackage{algorithm}
\usepackage{algpseudocode}
\usepackage{cite}
\usepackage{filecontents}
\usepackage{lipsum}
\usepackage{color}
\usepackage{esdiff}
\usepackage{epstopdf}
\usepackage{flushend}
\usepackage[utf8]{inputenc}
\usepackage[english]{babel}
\usepackage{amsthm}

\newtheorem{cor}{Corollary}[section]

\newtheorem{prop}{Proposition}[section]
 
\newtheorem{defn}{Definition}[section]

\theoremstyle{remark}

 \usepackage{soul}
 
 \usepackage[normalem]{ulem}

\begin{document}

\title{
\Huge {CT-CPP: Coverage Path Planning for 3D Terrain Reconstruction using Dynamic Coverage Trees}\\
}

\author{ \begin{tabular}{cccccccccc}
{Zongyuan Shen} & {Junnan Song} & {Khushboo Mittal} & {Shalabh Gupta$^\star$}
\end{tabular}\vspace{-24pt}




\thanks {The authors are with School of Engineering, Department of Electrical and Computer Engineering, University of Connecticut, Storrs, CT, 06269, USA.}

\thanks {$^\star$Corresponding author (e-mail:  shalabh.gupta@uconn.edu)}

\thanks{Digital Object Identifier (DOI): 10.1109/LRA.2021.3119870}

\thanks{\copyright  2021 IEEE. Personal use of this material is permitted. Permission from IEEE must be obtained for all other uses, in any current or future media, including reprinting/republishing this material for advertising or promotional purposes, creating new collective works, for resale or redistribution to servers or lists, or reuse of any copyrighted component of this work in other works.}

}

\markboth{IEEE Robotics and Automation Letters. Preprint Version. Accepted September 2021}
{Shen \MakeLowercase{\textit{et al.}}: CT-CPP: Coverage Path Planning for 3D Terrain Reconstruction using Dynamic Coverage Trees} 

\maketitle
\begin{abstract}
This letter addresses the 3D coverage path planning (CPP) problem for terrain reconstruction of unknown obstacle-rich environments. Due to sensing limitations, the proposed method, called CT-CPP, performs layered scanning of the 3D region to collect terrain data, where the traveling sequence is optimized using the concept of a coverage tree (CT) with a TSP-inspired tree traversal strategy. The CT-CPP method is validated on a high-fidelity underwater simulator and the results are compared to an existing terrain following CPP method. The results show that CT-CPP yields significant reduction in trajectory length, energy consumption, and reconstruction error. 
\end{abstract}

\vspace{-8pt}
\section{Introduction}
\label{sec:introduction}
\IEEEPARstart{C}{overage} path planning (CPP)~\cite{galceran2013} aims to find a path that enables a robot to scan all points in the search space with minimum  overlap~\cite{song2018}. CPP has a wide range of applications, such as seabed mapping~\cite{palomeras2018,shen2019}, spray-painting~\cite{vempati2018,atkar2005}, structural inspection~\cite{vidal2017,bircher2018receding,song2020online},  
mine-hunting~\cite{mukherjee2011}, oil spill cleaning~\cite{SGH13}, and arable farming~\cite{jin2011}.

Existing CPP approaches are of two types: 2D and 3D. While 2D approaches are applied for tasking on 2D surfaces~\cite{acar2002} (e.g., floor cleaning and lawn mowing), they are rendered insufficient for applications involving 3D surfaces. For example, a 2D CPP method can be applied for mapping a 3D underwater terrain by operating an autonomous underwater vehicle (AUV) at a fixed depth, such that the side-scan sonar can scan the seabed. However, this approach will be unable to explore the regions above the AUV, thus generating an incomplete terrain map. On the other hand, if the AUV is operated at a higher level, then the sensors will be unable to scan the terrain due to their limited field of view (FOV). Therefore, a 3D CPP method is needed for 3D terrain mapping; see Section~\ref{sec:review} for a review of 3D CPP methods.

\vspace{-3pt}
\subsection{Summary of the Proposed Method} This letter presents an online 3D CPP method for complete coverage of \textit{a priori} unknown environments with an application to underwater terrain reconstruction. In this method, a 3D underwater region is sliced into multiple 2D planes starting from the ocean surface, as shown in Fig.~\ref{fig:CoverageTreeExample}. The AUV can then perform top-down layered sensing by navigating on these planes and sensing the region below. However, it is possible that these planes are partitioned by obstacles into disconnected subregions, for example, levels $1$ and $2$ of Fig.~\ref{fig:CoverageTreeExample}. 
\begin{defn}[Disconnected Subregions]\label{define:Disconnectedness} Two subregions on a plane are said to be disconnected (or path-disconnected) if there exists no path on that plane connecting them. Note: disconnected subregions could be connected in the 3D space.
\end{defn}
\vspace{-3pt}

A simple layered sensing approach will miss the disconnected subregions which are on the other side of the obstacle and achieve only partial coverage. To address this issue, we developed a novel method, called CT-CPP, using a coverage tree (CT)~\cite{sadat2014}, where the nodes (except the root node) of CT represent the disconnected subregions, as shown in Fig~\ref{fig:CoverageTreeExample}. CT-CPP incrementally builds a CT in a top-down manner, which is used to plan and track the progress of 3D coverage.

\begin{figure}[t]
    \centering
    \includegraphics[width=1\columnwidth]{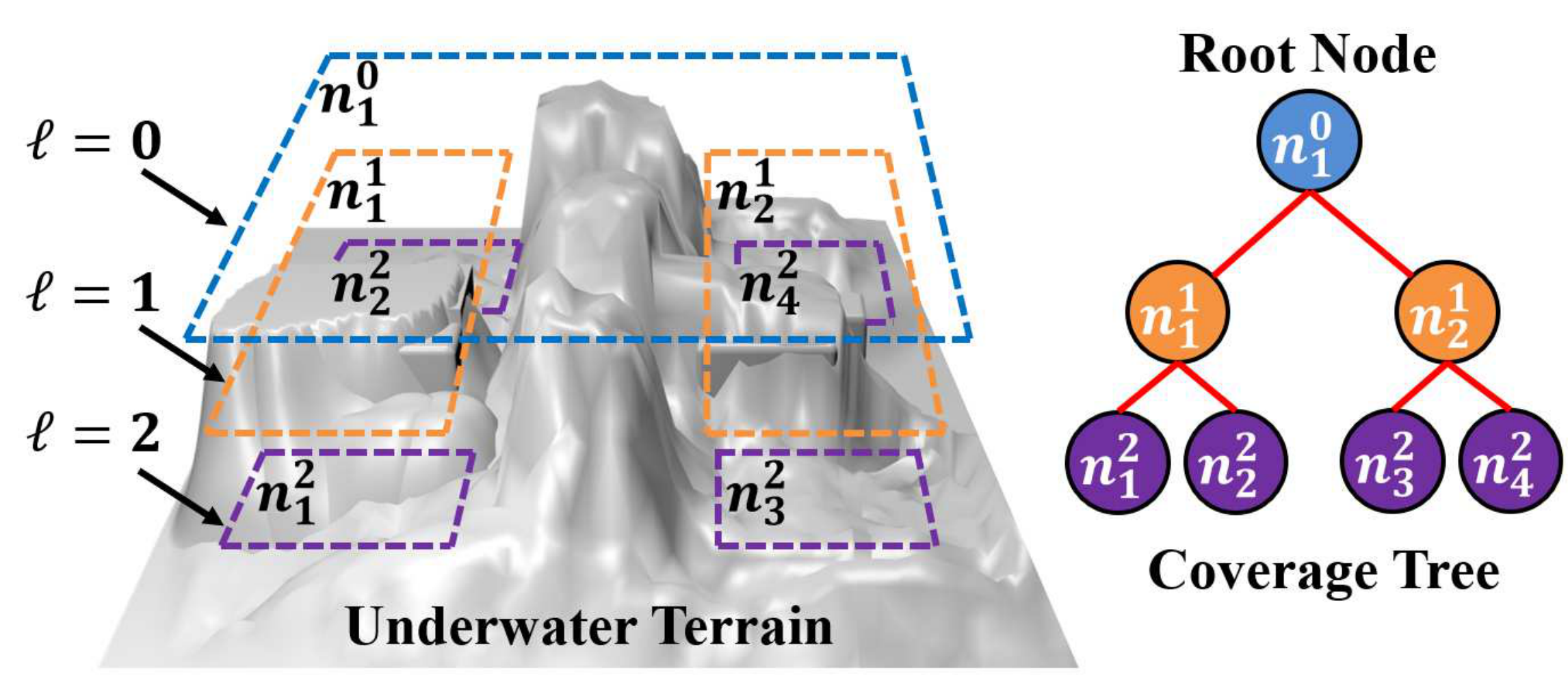}
    \caption{An example of a 3D region represented by a coverage tree.}
   \label{fig:CoverageTreeExample} \vspace{-16pt}
\end{figure}

Once the AUV reaches a node, it covers the corresponding subregion using a 2D CPP algorithm~\cite{song2018}. During the coverage of each planar subregion, the AUV uses the downward-facing multi-beam sonar sensor to collect data for the 3D terrain structures that are within the sensor's range extended at least up to the plane below. Based on the data, the AUV projects and stores the information about obstacles intersecting the plane below by forming a 2D probabilistic occupancy map (POM)~\cite{thrun2005}. Since the underwater terrain may contain narrow regions which can be risky for the AUV to navigate and avoid collisions, an image morphological operator `\textit{closing}'~\cite{soille2013} is applied on the POM to close the narrow areas. This serves two purposes: i) the updated 2D obstacle map is used to identify the disconnected subregions on that plane. i.e., it adds child nodes to the CT, and ii) it ensures the AUV's safety when it navigates on that plane. 

The updated tree is then used to plan the AUV trajectory by generating an optimized tree traversal sequence using a heuristics-based solution to the traveling salesman problem (TSP). The above process continues until the AUV completes the coverage of all nodes of the tree and no new nodes are created. Then, the data collected from all nodes are integrated offline for complete 3D terrain reconstruction. The performance of CT-CPP is comparatively evaluated on a high-fidelity underwater robotic simulator called UWSim~\cite{prats2012}. 

\begin{figure}[t]\vspace{-15pt}
    \centering 
    \subfloat[CT-CPP method]{      \includegraphics[width=0.50\columnwidth]{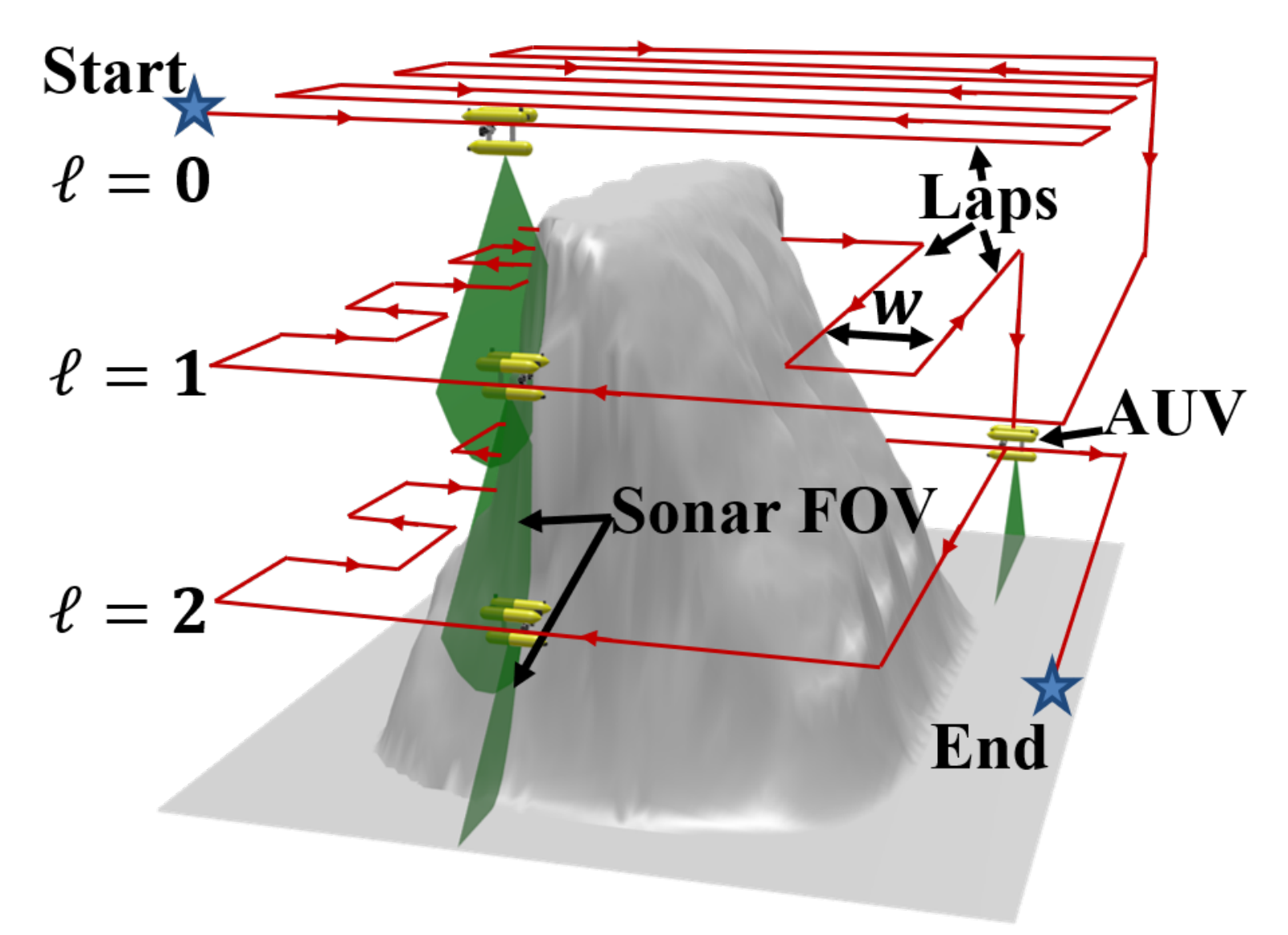}\label{fig:CTCPP}}
 \centering
    \subfloat[TF-CPP method]{
         \includegraphics[width=0.50\columnwidth]{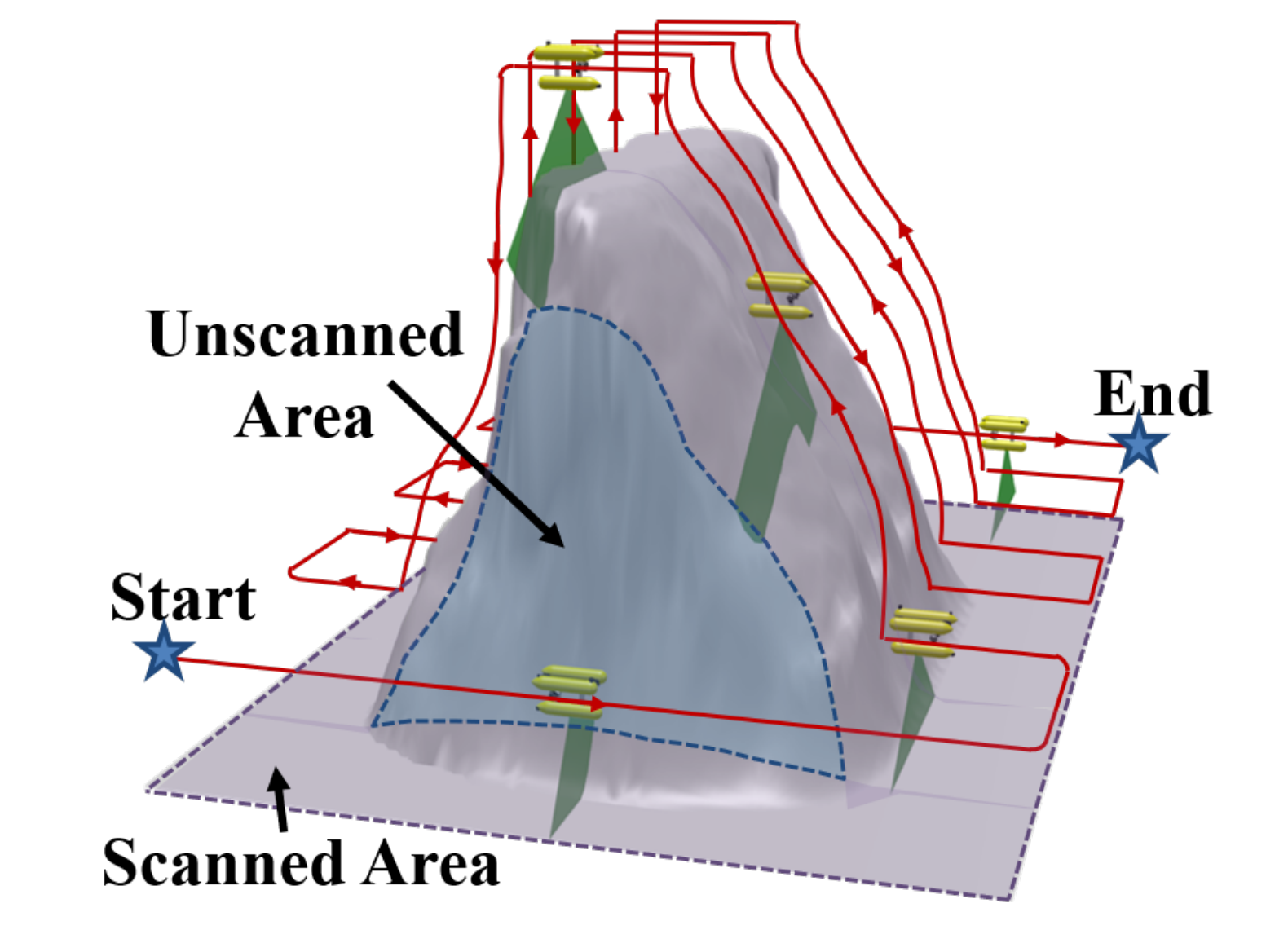}\label{fig:TFCPP}}
\caption{Coverage trajectories of (a) CT-CPP and (b) TF-CPP methods.}\label{fig:FundamentalFigure}
\vspace{-15pt}
\end{figure}

\vspace{-9pt}
\subsection{Contributions}\vspace{-3pt}
The main contribution of this letter is development of a novel online 3D CPP method, called CT-CPP, for unknown underwater terrain reconstruction considering limited sensing range, vehicle safety, and disconnected subregions formed by obstacles. CT-CPP follows a TSP-optimized node traversal sequence of an incrementally built CT for complete 3D coverage. The letter extends our preliminary work~\cite{shen2017} by presenting  i) a detailed formulation of the CT-CPP method, ii) generation of 2D symbolic maps for safe navigation on the planes, iii) a TSP-based tree traversal strategy, and iv) comparative evaluation showing significant reductions in trajectory length, energy consumption and reconstruction error.

\vspace{-9pt}
\subsection{Organization}\vspace{-3pt}
The letter is organized as follows. Section~\ref{sec:review} reviews the existing 3D CPP methods. While Section~\ref{sec:problem} formulates the problem, Section~\ref{sec:solution} describes the CT-CPP method. Section~\ref{sec:results} presents the results and Section~\ref{sec:conclusions} concludes the paper.

\vspace{-9pt}
\section{Related Work}
\label{sec:review}\vspace{-0pt}
A review of existing coverage methods is presented in~\cite{galceran2013}. While offline methods assume the environment to be \textit{a priori} known, online methods (i.e., sensor-based) are  adaptive and compute the coverage paths \textit{in situ} \cite{song2018, acar2002, shen2017}. Some relevant 3D CPP approaches are discussed below. 

Galceran et. al.~\cite{galceran2015} presented a method for underwater terrain reconstruction which relies on \textit{a priori} known bathymetric map. The terrain is classified into high-slope and planar areas, which are covered using a slicing algorithm and lawn-mowing paths, respectively. However, its performance can degrade if the \textit{a priori} information is incorrect. 
Lee et. al.~\cite{lee2009} decomposed a 3D underwater space into multiple 2D layers at various depths, and the AUV performs the 2D coverage~\cite{lumelsky1990} of each layer in a bottom-up manner. The explored areas at a lower layer are marked as artificial islands at higher layers, thus avoiding repeated scan of the same region; however, this bottom-up search could miss several disconnected regions.

Hert et. al.~\cite{hert1996} proposed an online 3D CPP method for underwater environments, called terrain following CPP (TF-CPP), where the AUV follows the back-and-forth paths systematically and maintains a constant distance to the terrain. Fig.~\ref{fig:FundamentalFigure}(b) shows an example of the TF-CPP trajectory. However, as seen in Fig.~\ref{fig:FundamentalFigure}(b), TF-CPP could miss the side faces of steep mountains. Moreover, this method did not address the risk issue in constricted areas. In addition, due to terrain following, it could generate  lot of vertical motions. In comparison, the CT-CPP trajectory, shown in Fig.~\ref{fig:FundamentalFigure}(a), will achieve complete 3D coverage (Defn~\ref{define:Complete_coverage}) by layered sensing, while ensuring safety in constricted areas. Table~\ref{tab:feature} shows a comparison of the key features of CT-CPP and TF-CPP.  

\begin{table}[!t]{}
\caption {{Comparison of Key Features of CT-CPP and TF-CPP} }\label{tab:feature}\vspace{-6pt}
\centering
 \begin{tabular}{l l l} 
 \toprule
  & {CT-CPP} & {TF-CPP} \\ 
 \hline \vspace{-4pt}
\tabincell{l}{{Terrain Surface}} &\tabincell{l}{Projectively planar}
 &\tabincell{l}{Projectively planar}\\ 
 \specialrule{0em}{3pt}{3pt}\vspace{-4pt}
 {Coverage Pattern} &\tabincell{l}{{Hierarchical multi-level}\\ {coverage of a 3D region}} & \tabincell{l}{{Terrain following}\\{and lapping}} \\
 \specialrule{0em}{3pt}{3pt} \vspace{-4pt}
 {Approach} &\tabincell{l}{{Uses coverage tree and}\\ {TSP-based optimization}\\{to plan the trajectory}}  & \tabincell{l}{{Maintains a safe distance}\\ {from the terrain using}\\ {sensor readings}}\\
 \specialrule{0em}{3pt}{3pt}\vspace{-4pt}
 \tabincell{l}{{Vehicle Safety}} & \tabincell{l}{{Probabilistic occupancy}\\{map to avoid constricted}\\{areas for safety}}   & \tabincell{l}{{No clear method to}\\ {avoid constricted areas}} \\
 \specialrule{0em}{3pt}{3pt}\vspace{-4pt}
 \tabincell{l}{{Energy Efficiency}} & \tabincell{l}{{Reduced vertical motions}}  & {Many vertical motions} \\
 \specialrule{0em}{3pt}{3pt}\vspace{-4pt}
 \tabincell{l}{{Coverage Quality}} & \tabincell{l} {{Low reconstruction error}}  & \tabincell{l}{{Missed side faces}} \\
 \specialrule{0em}{3pt}{3pt}\vspace{-0pt}
 \tabincell{l}{Control Effort} & \tabincell{l} {Easy to maintain straight\\  motion at a fixed depth}  & \tabincell{l}{ Difficult to follow\\  complex terrains} \\
 \toprule
 \end{tabular}
 \vspace{-20pt}
 \end{table}

Cheng et. al.~\cite{cheng2008} presented a 3D coverage strategy for urban structures using unmanned air vehicles (UAVs), where these structures are represented by simplified abstract models such as hemispheres and cylinders. In obstacle-free space, Sadat et. al.~\cite{sadat2014} presented a recursive algorithm for non-uniform coverage of unknown terrains using UAVs. The UAV operates at various altitudes, and identifies the areas of interest located at lower altitudes. This dynamically grows a coverage tree, where each node refers to an area of interest where high-resolution data is collected. Later in~\cite{sadat2015}, this method was improved using Hilbert space filling curves to reduce the trajectory length; however, their method works for obstacle-free environments. Some sampling-based approaches \cite{bircher2018receding,song2020online} are proposed for online exploration; however, they could generate random vertical motions consuming more energy.

\vspace{-10pt}
\section{Problem Formulation}
\label{sec:problem}
Let $\mathcal{U}\subset\mathbb{R}^3$ be an unknown underwater region containing a mountainous terrain. It is assumed that this terrain is a projectively planar surface~\cite{hert1996}. As shown in Fig.~\ref{fig:surface}, a projectively planar surface is intersected by a vertical line at only one point while a surface that is not projectively planar is intersected at three or more points. The objective is to map this terrain using an AUV equipped with a downward-facing sonar with a field of view (FOV) defined by the sensing range $r \in \mathbb{R}^+$ and aperture angle $\theta \in [0,\pi)$, as shown in Fig.~\ref{fig:DeltaDepth}. The AUV also has a forward-facing sonar for safe navigation on a plane.

For 3D CPP, $\mathcal{U}$ is sliced into $L\in \mathbb{N}^+$ equidistant horizontal planes $\{\mathcal{A}^\ell \subset \mathbb{R}^2, \ \ell = 0,\ldots L-1\}$, separated by a distance $\Delta h<r$, where $\mathcal{A}^0$ is the ocean surface, while $\mathcal{A}^{L-1}$ is the lowest plane from which the downward-facing sonar beams are able to reach the seabed. Note that these planes are constructed incrementally and $L$ is \textit{a priori} unknown. 

Starting from $\mathcal{A}^0$, the AUV navigates on each plane to collect data using its downward-facing sonar sensor within its FOV. This data is used offline for 3D terrain reconstruction~\cite{edelsbrunner1994}. Each plane is covered using a 2D-CPP algorithm (This paper uses the $\epsilon^*$-algorithm~\cite{song2018} for planar coverage). If there is an obstacle (e.g., an island) on the ocean surface, it can be detected by the forward-facing sonar when the AUV navigates on $\mathcal{A}^0$. The 2D-CPP algorithm generates a back and forth motion pattern that consists of laps as defined below.
\begin{defn}[Lap]
A lap is a straight line on which the AUV navigates without obstruction by an obstacle or a boundary during the back and forth coverage of a planar region. Furthermore, the lap width $w$ is defined as the distance between two adjacent laps. An example is shown in Fig.~\ref{fig:FundamentalFigure}a.
\end{defn}

Now we describe how $\Delta h$ is selected offline to ensure that the 3D space between any two successive planes is fully covered by the sonar's FOV.

\begin{prop}
Given the sensor FOV parameters $(r,\theta$), the 3D terrain between any two adjacent planes $\mathcal{A}^{\ell-1}$ and $\mathcal{A}^{\ell}$ is completely covered if $\Delta h$ is within the bounds 
\begin{equation*}\label{eq:depth}
0 < \Delta h \leq \left(\sqrt{r^2 - 2.25w^2} - 1.5w\cot(\theta/2)\right).\vspace{-10pt}
\end{equation*} 
\end{prop}
\begin{proof} Consider the side face of an obstacle. Suppose the AUV is navigating at $\mathcal{A}^{\ell-1}$. Since the sonar FOV with $\theta < \pi$ can provide only partial coverage of the 3D structures between $\mathcal{A}^{\ell-1}$ and $\mathcal{A}^{\ell}$, the sonar FOV from $\mathcal{A}^{\ell-2}$ must overlap with that of $\mathcal{A}^{\ell-1}$ to cover the gaps. This requires appropriate selection of $\Delta h$ for positioning $\mathcal{A}^{\ell-1}$ to achieve the minimum necessary overlap, as shown in Fig.~\ref{fig:delta_a}. Now, suppose for safety consideration, the minimum distance of the closest lap from the obstacle is $0.5w$. Then, the maximum distance the obstacle could be from this lap is $1.5w$, otherwise the AUV could move to a closest lap to the obstacle. Then, considering the worst case scenario when the distance between the closest lap and obstacle is $1.5w$, we get from the geometry of Fig.~\ref{fig:delta_a}, $\left(\Delta h + 1.5w\cot(\theta/2)\right)^2 \leq r^2 - (1.5w)^2$. Similarly, for the flat face of an obstacle, only distance of $0.5w$ needs to be considered. From  Fig.~\ref{fig:delta_b}, we see that the above bound for $\Delta h$ will also guarantee coverage of horizontal surfaces. 
\end{proof}\vspace{-6pt}
\begin{cor} \textrm{For} $\Delta h >0$, $w<\frac{2}{3}r\sin(\theta/2)$.
\end{cor}

\begin{figure}[t]
    \centering
    \subfloat[Projectively planar surface.]{
        \includegraphics[width=0.48\columnwidth]{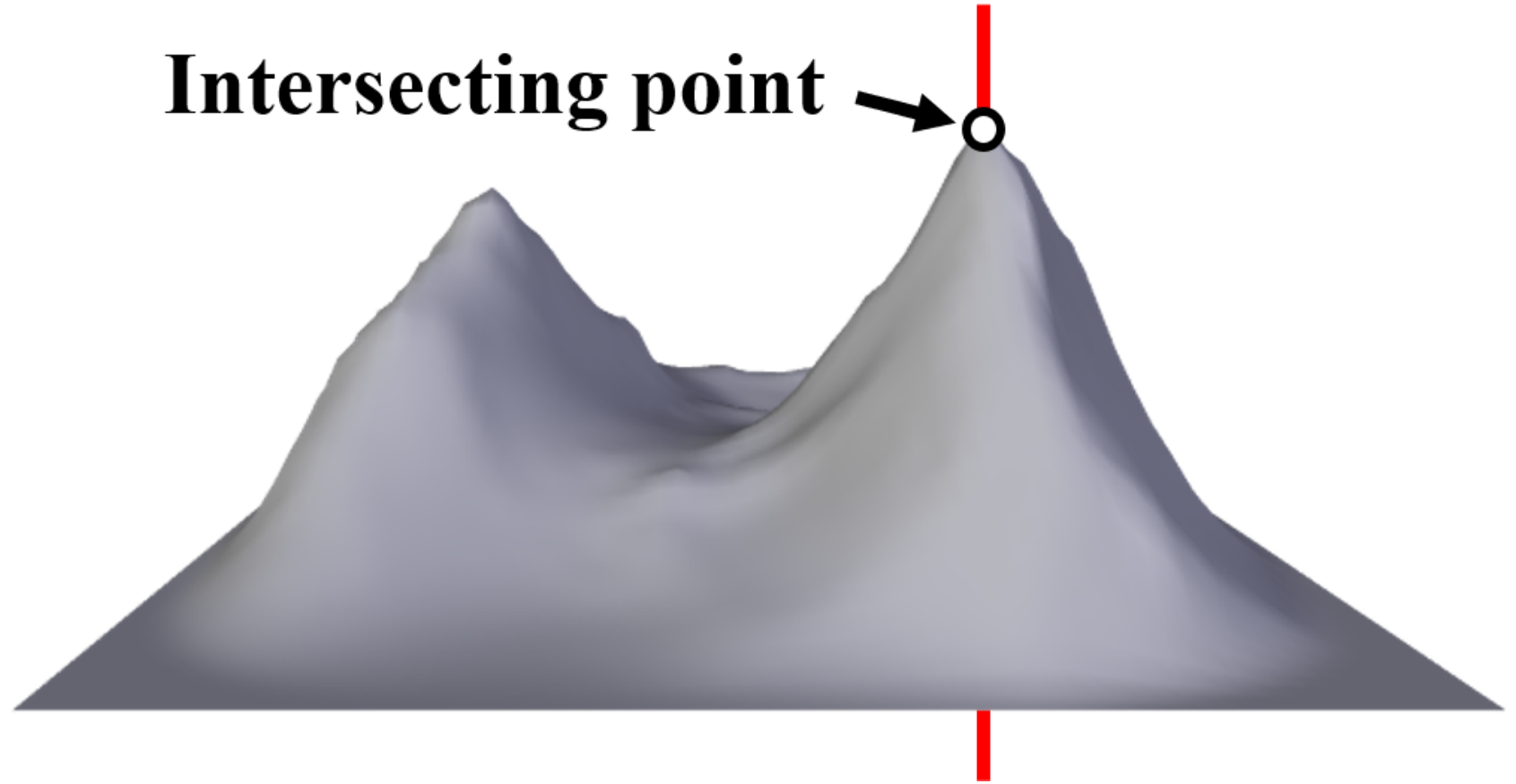}\label{fig:type1}}
 \centering
    \subfloat[Not a projectively planar surface.]{
         \includegraphics[width=0.48\columnwidth]{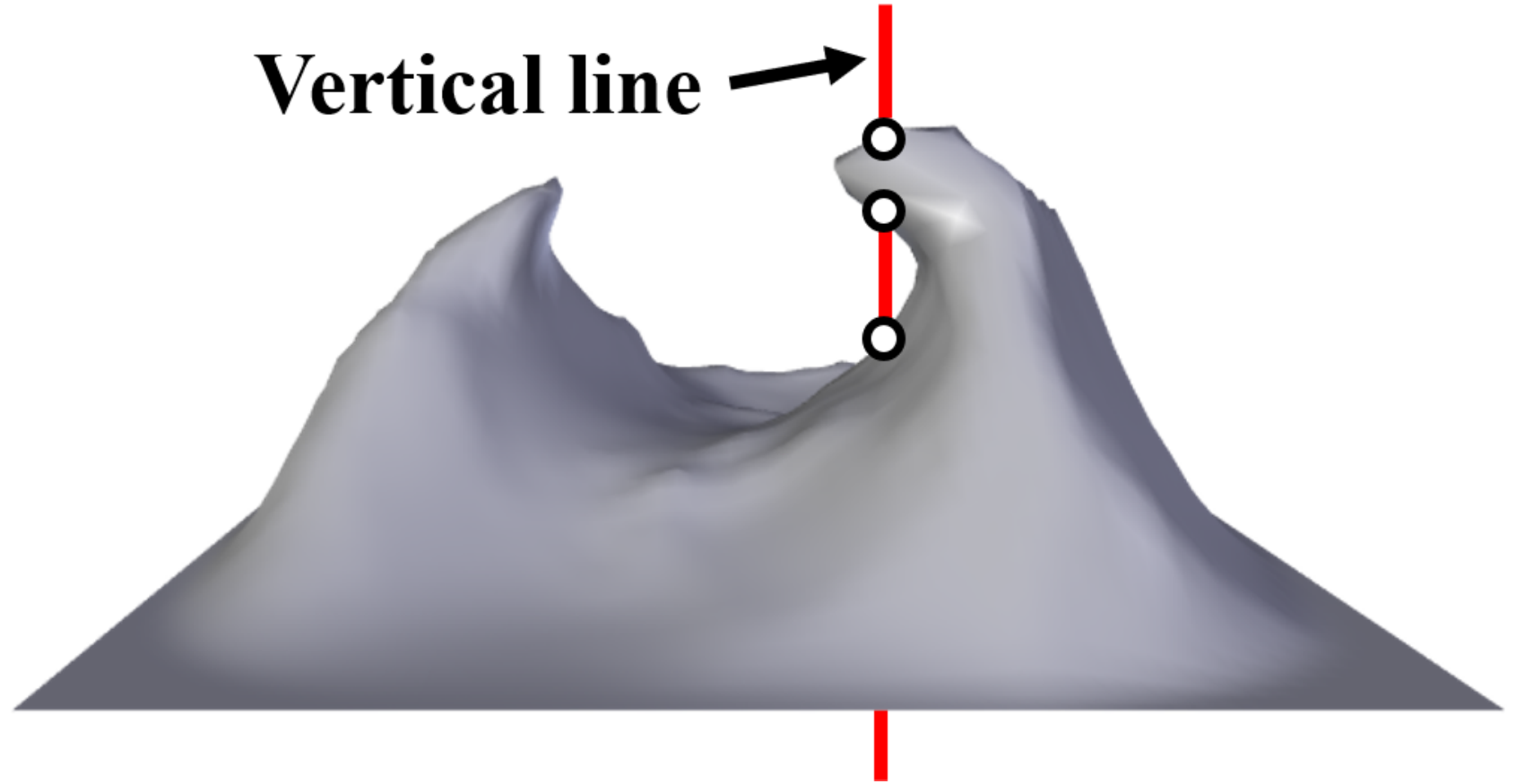}\label{fig:type2}} 
          \caption{Underwater terrain surface types.}\label{fig:surface}
\vspace{-15pt}
\end{figure}

While navigating on a plane, the AUV may detect disconnected subregions on the plane below. The process of identifying disconnected subregions is described  in Section \ref{SafetyMap}.   
Let the set of disconnected subregions on a plane $\mathcal{A}^{\ell}$ be denoted as $\mathcal{A}^{\ell}_{\mathcal{S}}$= $\{\mathcal{A}^{\ell}_{s}\subseteq \mathcal{A}^{\ell}, s=1,\ldots |\mathcal{A}^{\ell}_{\mathcal{S}}|\}$, where each $\mathcal{A}^{\ell}_s$ is a disconnected subregion such that:

\begin{itemize}
\item $\mathcal{A}^{\ell}_{i}\cap\mathcal{A}^{\ell}_{j}=\emptyset, \forall \mathcal{A}^{\ell}_{i},\mathcal{A}^{\ell}_{j} \in \mathcal{A}^{\ell}_{\mathcal{S}},i\neq j$ \textrm{and}\\
\item  $\bigcup_{s=1}^{|\mathcal{A}^{\ell}_{\mathcal{S}}|} \mathcal{A}^{\ell}_{s} \subseteq \mathcal{A}^{\ell}$ 
\end{itemize} 
Then, the total area formed by all such subregions is: 
\begin{equation}
\mathcal{A}_C = \bigcup_{\ell=0}^{L-1} \bigcup_{s=1}^{|\mathcal{A}^{\ell}_{\mathcal{S}}|} \mathcal{A}^{\ell}_{s}
\end{equation}

\begin{figure}[t]
    \centering
    \subfloat[Complete coverage of a vertical surface]{
        \includegraphics[width=0.96\columnwidth]{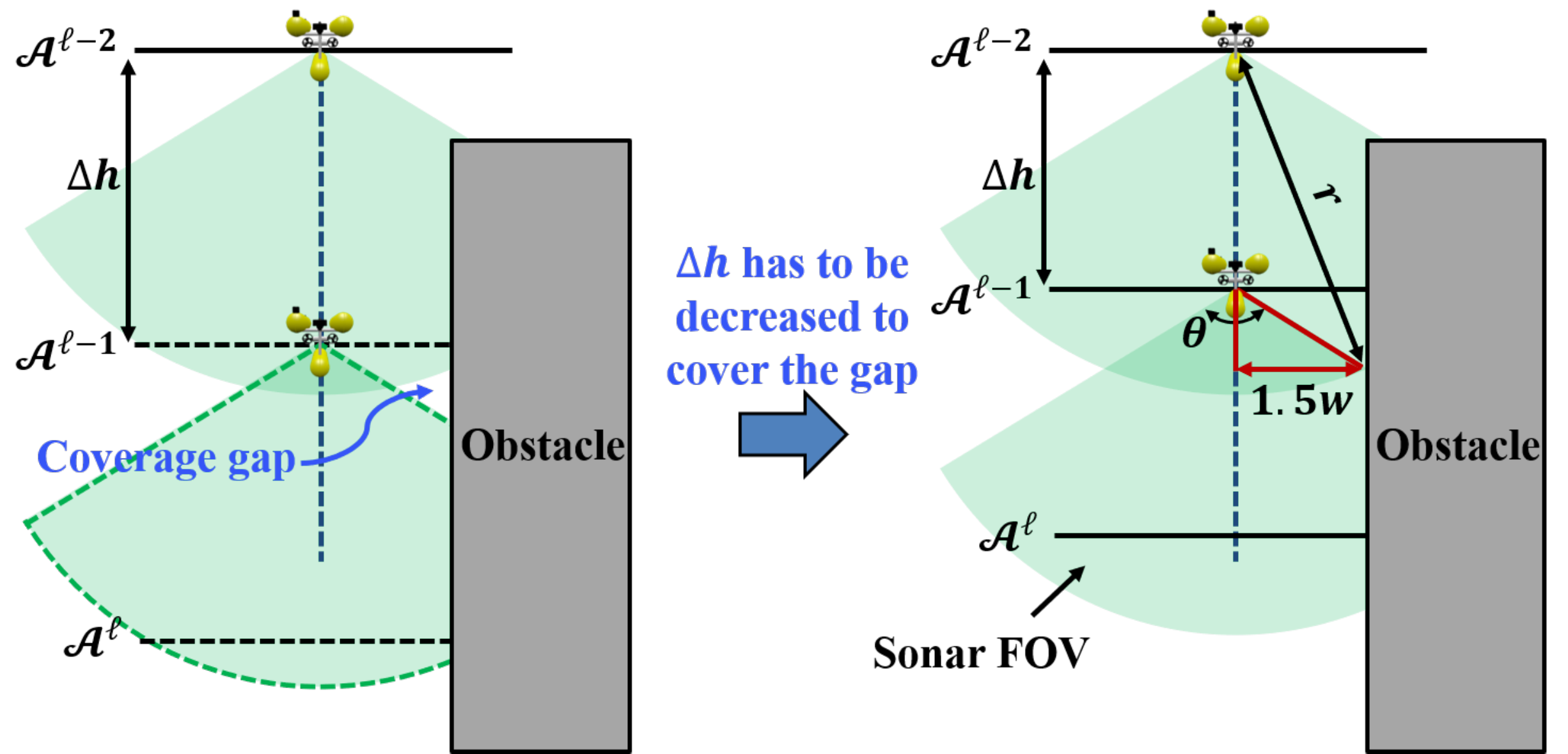}\label{fig:delta_a}}\vspace{3pt}
 \centering
    \subfloat[Complete coverage of a horizontal surface]{
         \includegraphics[width=0.96\columnwidth]{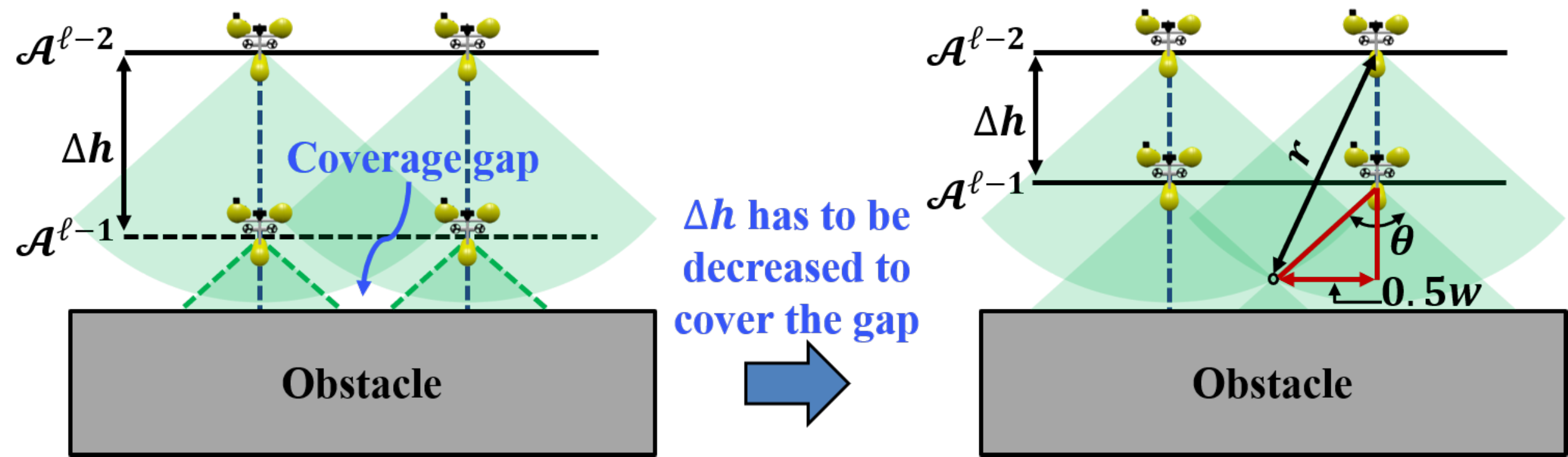}\label{fig:delta_b}}\vspace{-3pt}
         
          \caption{Setting $\Delta h$ for complete coverage of 3D space between any two successive planes $\mathcal{A}^{\ell-1}$ and $\mathcal{A}^{\ell}$ with limited sonar FOV.}\label{fig:DeltaDepth}
\vspace{-15pt}
\end{figure}

Thus, complete coverage is achieved if each of these subregions are visited and covered. 
\begin{defn}[Complete Coverage]\label{define:Complete_coverage}
Let $\mathcal{A}(k)\in \mathcal{A}^{\ell}_{\mathcal{S}}$ be the subregion visited and covered by the AUV at time instance $k$. Then, $\mathcal{U}$ is said to achieve complete coverage if $\exists K \in \mathbb{N}^{+}$ s.t. the sequence $\{\mathcal{A}(k),k=1,\ldots,K\}$ covers $\mathcal{A}_C$, i.e., 
\begin{equation}
\mathcal{A}_C = \bigcup_{k=1}^{K} \mathcal{A}(k)
\end{equation}
\end{defn}

Next, we present a 3D CPP method for complete coverage of $\mathcal{A}_C$, and thereby achieving full terrain reconstruction of $\mathcal{U}$. 

\vspace{-6pt}
\section{CT-CPP Method}
\label{sec:solution}

This section presents the details of the CT-CPP method.

\vspace{-10pt}
\subsection{Coverage Tree}\label{coverage_tree}
A coverage tree is used to track the search progress and  compute the sequence of visiting subregions during the coverage operation. Formally, a coverage tree is defined as follows.

\begin{defn}[Coverage Tree]\label{define:coverage_tree}
A coverage tree $\mathcal{Q}=\left(\mathcal{N},\mathcal{B}\right)$ is defined as an undirected acyclic graph that consists of: 
\begin{itemize}
\item A node set $\mathcal{N} = \left\{n^\ell_s: s = 1,\ldots|\mathcal{A}^{\ell}_{\mathcal{S}}|; \ell=0,\ldots L-1\right\}$, where a node $n^\ell_s$ corresponds to the subregion $\mathcal{A}^\ell_s \in \mathcal{A}^{\ell}_{\mathcal{S}}$ at level $\ell$. While  $n^0_1$ is the single root node at level $0$, every other node has only one parent.
\item A branch set $\mathcal{B}=\left\{b_i: i=1,\ldots\left|\mathcal{B}\right|\right\}$, where each branch connects a parent node with its child node. 
\end{itemize}
\end{defn}

While exploring a node (i.e., a planar subregion) at any level, the CT is incrementally built online by adding the child nodes corresponding to the disconnected subregions at the level below. This requires construction of the 2D symbolic map for the level below to identify the disconnected subregions on that plane as described below.  

\vspace{-10pt}
\subsection{Generation of 2D Symbolic Maps for Navigation on Planes}\label{SafetyMap}

While exploring any plane $\mathcal{A}^{\ell-1}$, $\ell \in \{1,\ldots L-1\}$, the AUV constructs the probabilistic occupancy map (POM) for $\mathcal{A}^{\ell}$, i.e., the plane below, using sensor measurements. 

\subsubsection{Construction of the POM} For constructing POM, first we construct a tiling on $\mathcal{A}^{\ell}$ as follows.

\begin{defn}[Tiling]\label{define:Tiling}
A set $\mathcal{T}^{\ell} = \{\tau^{\ell}_{\gamma} \subset \mathbb{R}^2:{\gamma}= 1,\ldots|\mathcal{T}^{\ell}|\}$ is a tiling of $\mathcal{A}^{\ell}$ if its elements, called tiles (or cells) have mutually exclusive interiors and cover $\mathcal{A}^{\ell}$, i.e., 
\begin{eqnarray*}
& \bullet & \ I\big(\tau^{\ell}_{\gamma}\big) \bigcap I\big(\tau^{\ell}_{\gamma'}\big) =\emptyset, \forall {\gamma},{\gamma'} \in \{1,\ldots|\mathcal{T}^{\ell}|\}, {\gamma} \neq {\gamma'}\\
& \bullet & \ \bigcup_{\gamma=1}^{|\mathcal{T}^{\ell}|}\tau_{\gamma} =\mathcal{A}^{\ell},
\end{eqnarray*}
where $I\big(\tau^{\ell}_{\gamma}\big)$ denotes the interior of the cell $\tau^{\ell}_{\gamma} \in \mathcal{T}^{\ell}$. 
\end{defn}

The POM stores the probability of obstacle occupancy at each cell of the tiling at $\mathcal{A}^{\ell}$.  Given the sensor information collected by the AUV while navigating at plane $\mathcal{A}^{\ell-1}$, the occupancy grid mapping algorithm~\cite{thrun2005} is used to estimate the occupancy probability for each cell $\tau^{\ell}_{\gamma}$ on plane $\mathcal{A}^{\ell}$, to generate its probabilistic occupancy  map (POM). Let $o^{\ell}_{\gamma}$ be a random variable defined on the set \{0,1\} to model the occupancy of $\tau^{\ell}_{\gamma}$, where $0$ and $1$ denote the obstacle-free and obstacle-occupied cells, respectively. Due to lack of a priori knowledge of the environment, all cells are initialized with a probability of $0.5$. Subsequently, the occupancy probability is updated using the Bayes' rule as follows
\begin{equation}\label{eq:ogm}
p(o^{\ell}_{\gamma}|z_{1:t},x_{1:t})=\frac{p(z_t|o^{\ell}_{\gamma}, z_{1:t-1},x_{1:t})p(o^{\ell}_{\gamma}|z_{1:t-1},x_{1:t})}{p(z_t|z_{1:t-1},x_{1:t})},
\end{equation}
where $z_{1:t}$ and $x_{1:t}$ denote the set of sensor measurements and set of robot positions, respectively; from the beginning until time $t$. Based on the Markov assumption and converting into the log odds notation~\cite{thrun2005} we get the recursive relation
\begin{multline}\label{eq:ogm_recursive}
l(o^{\ell}_{\gamma}|z_{1:t},x_{1:t})=\underbrace{l(o^{\ell}_{\gamma}|z_t,x_t)}_\text{inverse sensor model}+\underbrace{l(o^{\ell}_{\gamma}|z_{1:t-1},x_{1:t-1})}_\text{recursive term} ,
\end{multline} 
where $l(x)=log\frac{p(x)}{1-p(x)}$. Since the prior $p(o^{\ell}_{\gamma}) = 0.5$, we have $l(o^{\ell}_{\gamma})=0$. Also,  $l(o^{\ell}_{\gamma}|z_t,x_t)=\sum_{m=1}^{M}l(o^{\ell}_{\gamma}|z_t^m,x_t)$, where $m\in \mathbb{N}^+$ denotes the sensor beam, while $M\in \mathbb{N}^+$ represents the total number of beams in the multi-beam sonar. The sampling interval is $\Delta t \in \mathbb{R}^+$. If the $m^{th}$ beam detects an obstacle in $\tau^{\ell}_{\gamma}$, then $l(o^{\ell}_{\gamma}|z_t^m,x_t)$ is equal to $l_{occ}$, else it is equal to $l_{free}$, where $l_{free}=-l_{occ}$. If the beam does not even pass through the cell, then $l(o^{\ell}_{\gamma}|z_t^m,x_t)=0$. Next, $l_{occ}$ and $l_{free}$ are computed as follows. Specifically, the total number of beams crossing a certain cell per scan is given as $B=\frac{2Mtan^{-1}(\frac{w}{2\Delta h})}{\theta}$. Then, the total number of beams crossing the cell during its traversal by the AUV is $B_{total}=\frac{w/v}{\Delta t}B$, where $v$ is the AUV speed. Now, if all beams detect an obstacle, then considering the effects of false measurements, we assume that a probability of $0.9$ is achieved about the cell's obstacle occupancy. Then, $ B_{total} \times l_{occ}=log\frac{0.9}{1-0.9}$. Thus, $l_{occ}=0.002$.

\vspace{6pt}
\subsubsection{Construction of the Symbolic Map} The POM is then transformed into a symbolic map of $\mathcal{A}^{\ell}$ using a symbolic encoding~\cite{GRP09} $\Phi^{\ell}: \mathcal{T}^{\ell} \rightarrow \Sigma$, which reads the probability of each cell $\tau^{\ell}_{\gamma} \in \mathcal{T}^{\ell}$ and assigns it a state from the alphabet $\Sigma=\{U,S,T\}$, where $U\equiv \text{Unexplored}$, $S\equiv \text{Safe}$, and $T \equiv \text{Threat}$. While $U$ is assigned to the cells which have not been scanned by the AUV sensor, $T$ is assigned to the cells with high occupancy probability, thus posing a risk to the AUV. The remaining cells are assigned $S$, and are considered safe for the AUV navigation. 
To compute the $T$ and $S$ cells,  first an image morphological operator `\textit{closing}'~\cite{soille2013} is applied to the POM to close constricted spaces for safety of the AUV. The closing operator expands the boundaries of regions with high occupancy probability cells and shrinks those with low occupancy probability cells by performing image dilation followed by erosion. 
Fig.~\ref{fig:ImageClosingExample} shows an example where the closing operator enlarges the regions in constricted spaces with high occupancy probability. Subsequently, a symbolic map $\Phi^{\ell}$ is obtained such that the cells whose occupancy probability is higher than the threat probability $p_{T}=0.2$ are marked as threat ($T$), while the others are marked as safe ($S$). This means that the 2D coverage planner at that plane will not allow the AUV to go too close to the risky regions for vehicle safety.

\begin{figure}[t]
    \centering
    \includegraphics[width=0.9\columnwidth]{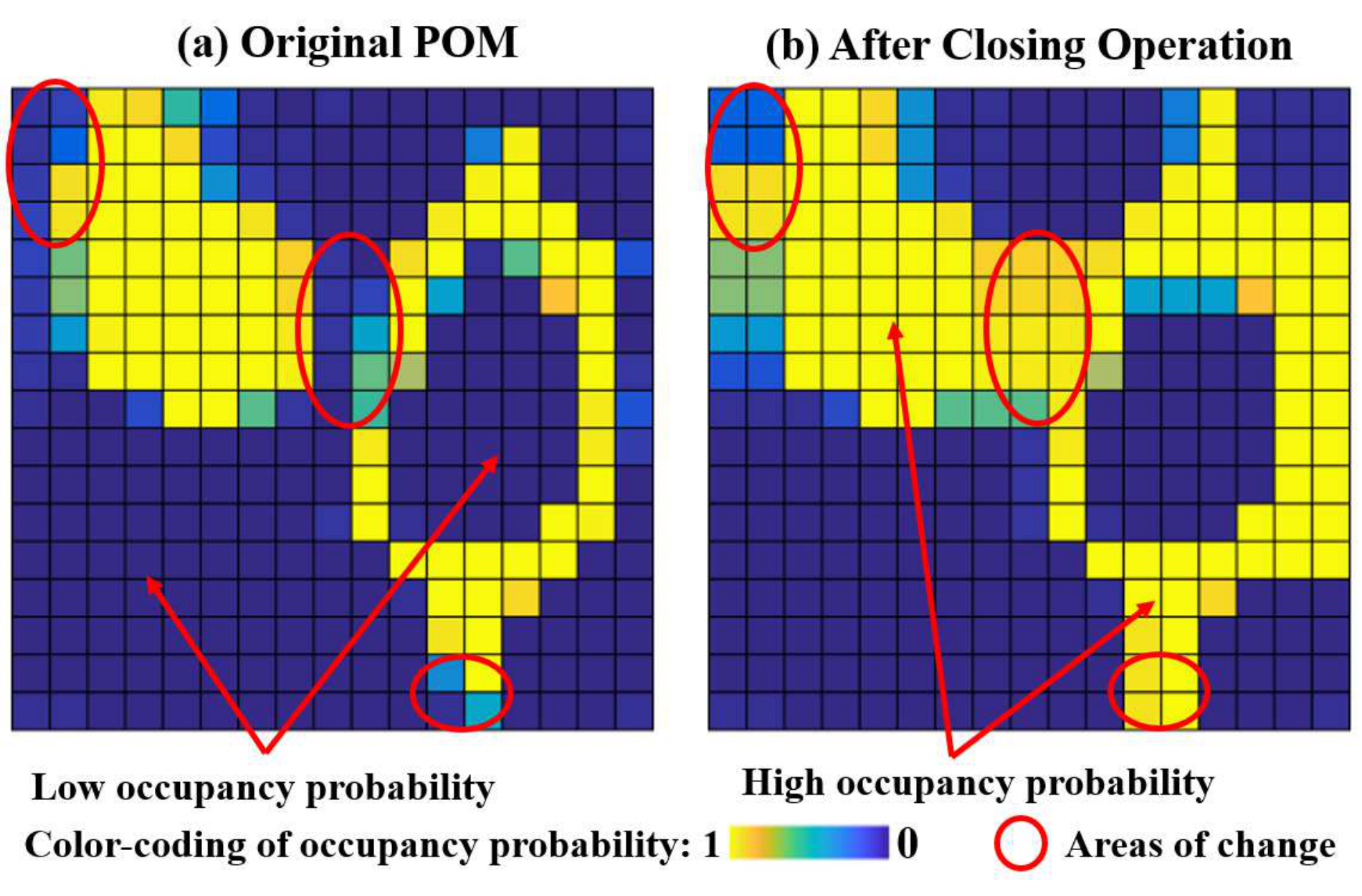} 
    \caption{The Probabilistic Occupancy Map (POM) before and after the image morphological closing operation with $3\times 3$ structure element.}
    \label{fig:ImageClosingExample} \vspace{-15pt}
\end{figure}

\vspace{3pt}
\subsubsection{Identification of Disconnected Subregions} Once the symbolic map is obtained,
the \textit{floodfill} algorithm \cite{Heckbert1990} is applied to identify the set of disconnected subregions $\mathcal{A}^{\ell}_{\mathcal{S}}$. Specifically, starting at a safe cell which does not belong to any identified subregion, the algorithm recursively searches for reachable safe cells in four directions until no more safe cells can be found. Then, these cells are grouped together to form a new subregion. The above process repeats until all subregions are found. These subregions are then added as child nodes to the current node of the CT.

\begin{figure*}[t]
    \centering
    \subfloat{\includegraphics[width=1\textwidth]{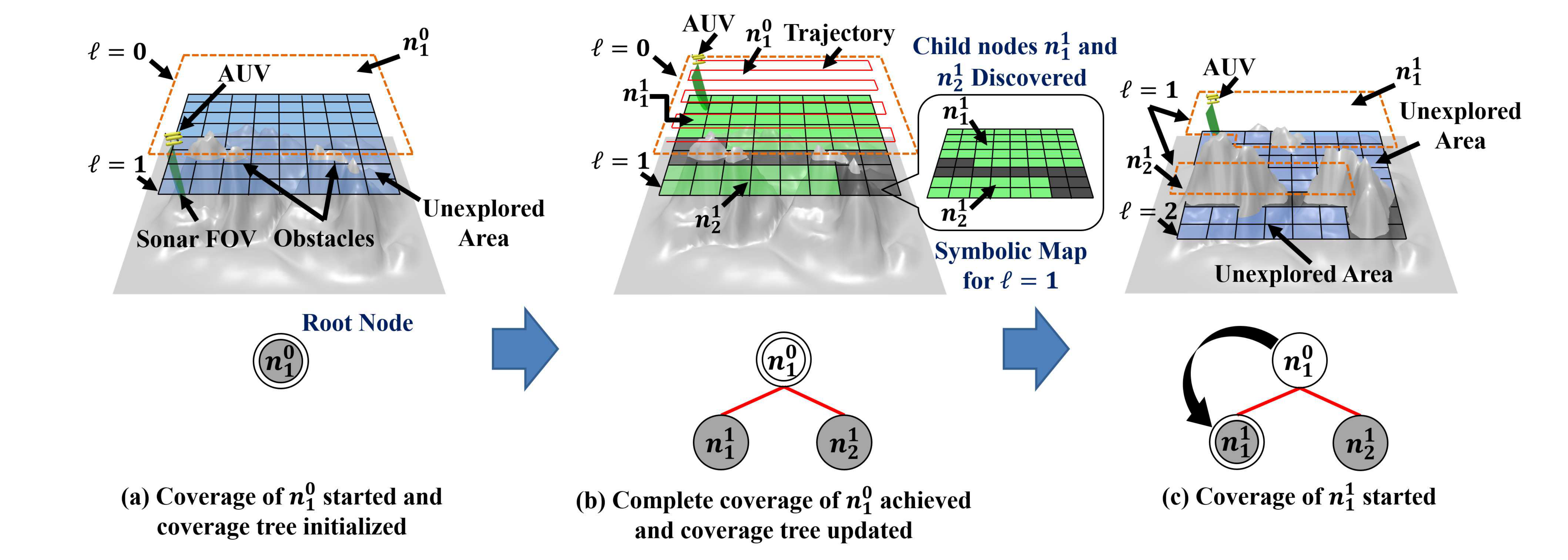}\label{fig:Tree:a}}
    \hspace{\fill}
    \centering
    \subfloat{\includegraphics[width=1\textwidth]{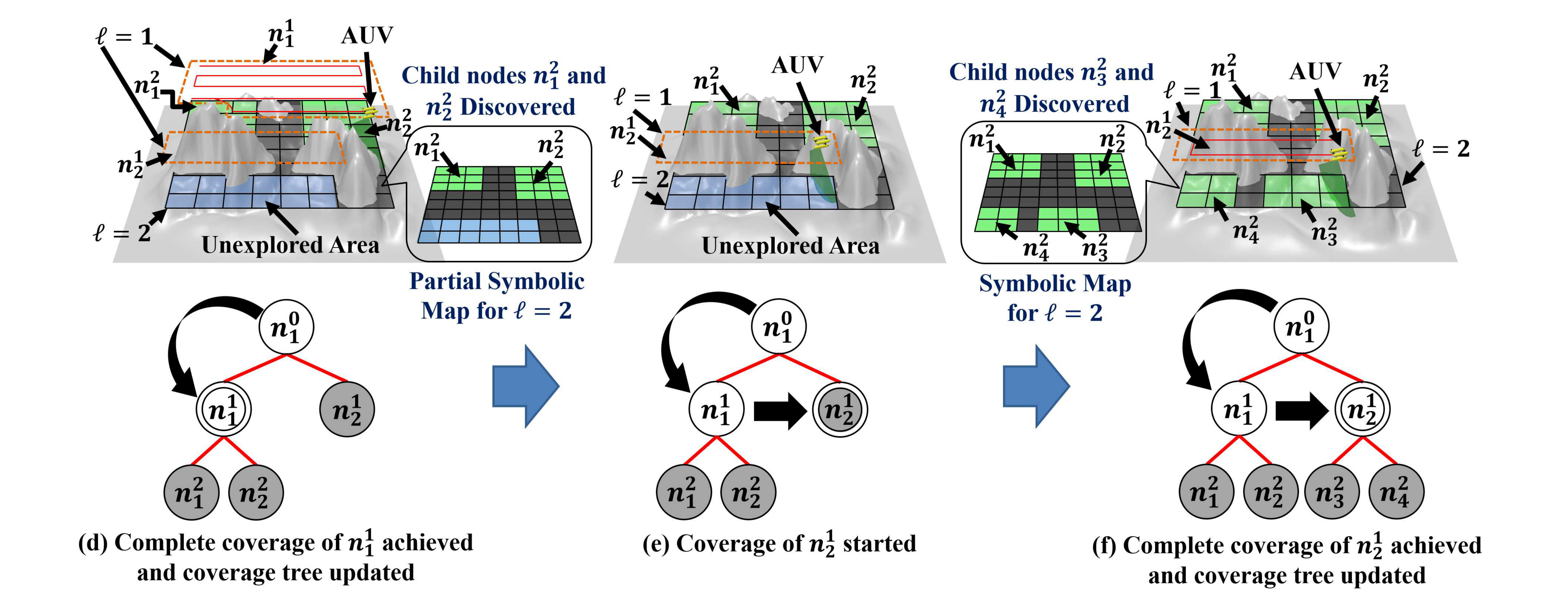}\label{fig:Tree:b}}
    \hspace{\fill}
    \centering
    \subfloat{\includegraphics[width=1\textwidth]{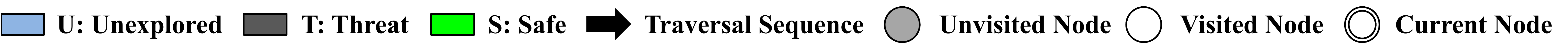}\label{fig:Tree:c}}
    \hspace{\fill}
\caption{Incremental construction of the coverage tree.}
\label{fig:CoverageTree} \vspace{-15pt}
\end{figure*}

\vspace{-11pt}
\subsection{Incremental Construction of the Coverage Tree}\label{CoverageTree}
\vspace{-3pt}
Fig.~\ref{fig:CoverageTree} illustrates the process of incremental construction of the CT. As seen in Fig.~\ref{fig:CoverageTree}(a), the tree $\mathcal{Q}$ is first initialized with the root node $n_1^0$ which corresponds to the search area $\mathcal{A}_1^0$ at level $\ell = 0$. The AUV covers $n_1^0$ using a 2D-CPP algorithm~\cite{song2018}, which guarantees complete coverage of any 2D connected region. While navigating $\mathcal{A}_1^0$, AUV uses the downward-facing multi-beam sonar sensors to generate the POM for level $\ell = 1$. Upon complete coverage of $n_1^0$, the POM is transformed into a symbolic map (see details in Section~\ref{SafetyMap}) as shown in Fig.~\ref{fig:CoverageTree}(b). The symbolic map reveals that at level $\ell = 1$, there exist two disconnected subregions separated by obstacles. These newly discovered subregions are updated to the CT as child nodes $n^1_1$ and $n^1_2$. Then the AUV covers the node $n^1_1$ and adds its children $n^2_1$ and $n^2_2$ to the tree, as shown in Figs.~\ref{fig:CoverageTree}(c) and (d). Similarly, the AUV covers $n^1_2$ and adds its children $n^2_3$ and $n^2_4$, as shown in Figs.~\ref{fig:CoverageTree}(e) and (f). The operation stops when there are no unvisited nodes available in the tree and $\mathcal{A}_C$ is completely covered. 

\vspace{-15pt}
\subsection{Computation of the Tree Traversal Sequence}
\label{traversal_strategy}
With the incremental construction of the CT, a modified TSP-based strategy is proposed for tree traversal.

Let $e_k, \ k \in \{1,\ldots|\mathcal{N}|\}$ denote the event that the AUV finishes coverage of the $k^{th}$ node on the CT. Let $\mathcal{Q}_{e_k}$ be the tree that has been updated right after the event $e_k$. Let $\mathcal{N}_{e_k}$ be the node set of $\mathcal{Q}_{e_k}$. For the purpose of tracking the coverage progress, each node of $\mathcal{Q}_{e_k}$ is assigned a state using another symbolic encoding $\Phi_{e_k}:\mathcal{N}_{e_k}\rightarrow \{E,U\}$, where $E\equiv \text{Explored}$ and $U\equiv \text{Unexplored}$. This encoding generates the set partition $\mathcal{N}_{e_k}=$$\{\mathcal{N}_{E},\mathcal{N}_{U}\}$, where $\mathcal{N}_{E}$ and $\mathcal{N}_{U}$ are the sets of explored and unexplored nodes, respectively. 

Next, a graph $\mathcal{G}=(\mathcal{V},\mathcal{E})$ is derived from $\mathcal{Q}_{e_k}$, whose vertex set $\mathcal{V} = \{v_i: v_0\in \mathcal{N}_{E} \ \textrm{and} \ v_i\in \mathcal{N}_{U}, \ \forall i=1,\ldots\eta-1\}$, where $v_0$ is the most recently explored node, and $|\mathcal{V}|=\eta$. The edge set $\mathcal{E} = \left\{e_{ij}\equiv\left(v_i,v_j\right): v_i \neq v_j, \forall v_i,v_j \in \mathcal{V} \right\}$.  

Furthermore, to compute the transition cost between vertices, each vertex $v_i \in \mathcal{V}$ is assigned coordinates $(x_{v_i},y_{v_i},z_{v_i})$, such that for $v_0$ they denote the current position of the AUV, while for other nodes, they denote the centroid of the corresponding subregion. Consider any two nodes $v_i$ and $v_j$ located at planes $\ell_i$ and $\ell_j$, respectively. Further, suppose that the first common ancestor node of $v_i$ and $v_j$ on the CT, from which a unique path exists to $v_j$, is located at level $\ell_a$. Then, the edge $e_{ij}\in\mathcal{E}$ between $v_i$ and $v_j$ is assigned an estimated transition cost $w_{ij} \equiv w_{v_iv_j} \in {\mathbb{R}^+}$, as  
\begin{equation}\label{eq:weight} 
w_{ij}=w_{ij}^{up}+w_{ij}^{hz}+w_{ij}^{down} 
\end{equation}
where $w_{ij}^{up}=\left|h(\ell_i)-h(\ell_a)\right|$ corresponds to the vertical distance between the planes of $v_i$ and the ancestor node; and $w_{ij}^{down}=\left|h(\ell_j)-h(\ell_a)\right|$ is the vertical distance between the planes of this ancestor node and $v_j$. And $w_{ij}^{hz}$ is the heuristic (assuming obstacle-free) horizontal transition cost  between two centroids. Note that the ideal way is to consider different combinations of entry and exit points as well as the associated coverage path length in each node. However, this could increase computational time. Thus, we employed the following heuristic cost method for simplicity. \vspace{-6pt}
\begin{equation}\vspace{-3pt}
w_{ij}^{hz}=\parallel \left(x_{v_i},y_{v_i}\right)-\left(x_{v_j},y_{v_j}\right) \parallel_2
\label{eq:leftweight} \vspace{-3pt}
\end{equation}
Using these costs, the weight matrix is obtained as      
$\mathcal{W}=\left[w_{ij}\right]_{\eta \times \eta}$, 
where $w_{ii} = 0, \forall i\in \left\{0,\ldots\eta-1\right\}$. 
The optimal trajectory is equivalent to the shortest path starting with $v_0$ and covering each vertex once without returning to $v_0$.

\begin{figure}[t]
    \centering
    \includegraphics[width=0.98\columnwidth]{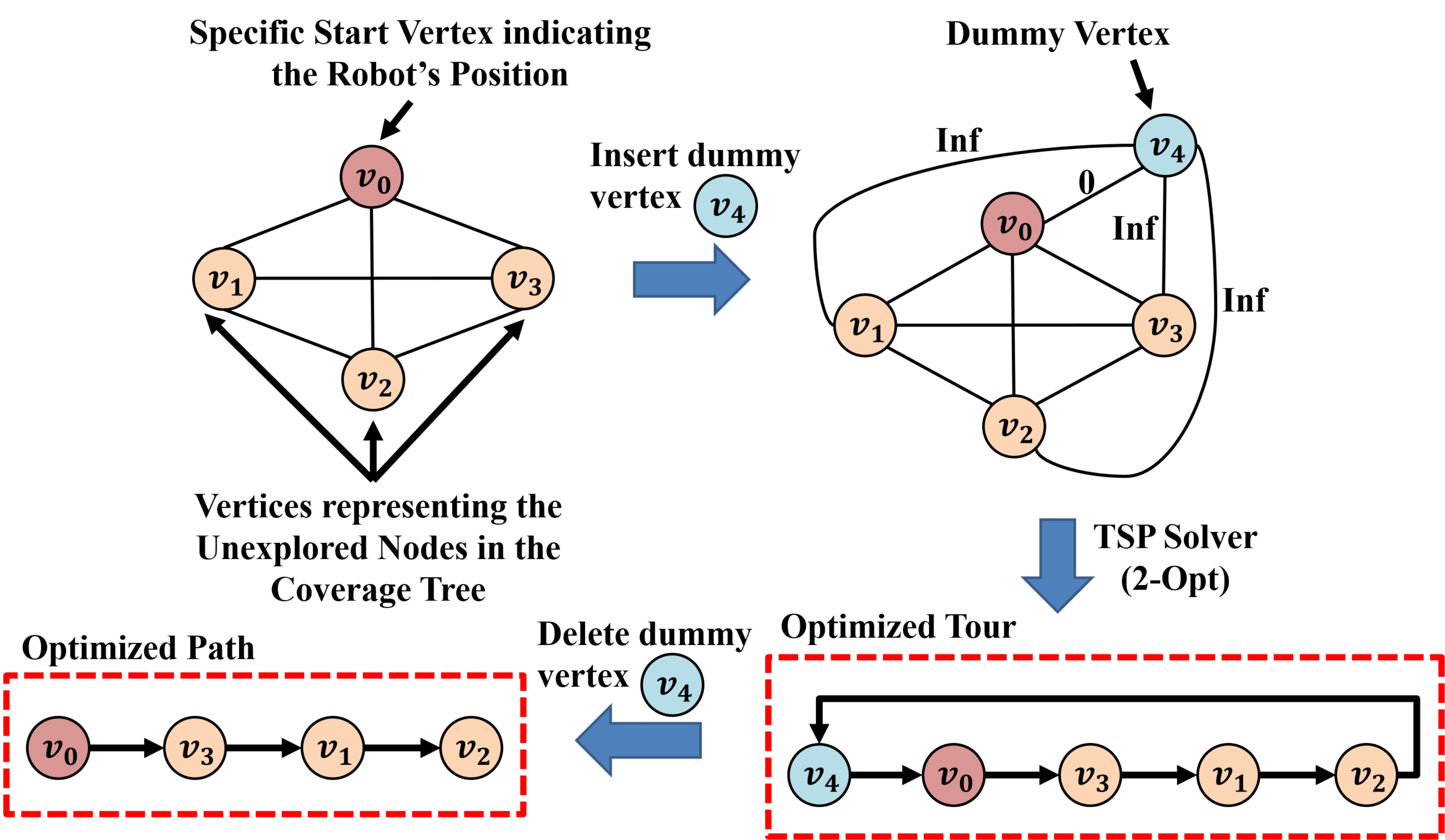}\vspace{-3pt}
    \caption{An example of the solution procedure.}
    \label{fig:TSP}\vspace{-15pt}
\end{figure}

This problem can be transformed into the Traveling Salesman Problem (TSP)\cite{aarts2003} which finds the shortest path for visiting all vertices and returning back to the start vertex. To do so, the vertex set $\mathcal{V}$ is expanded to $\mathcal{V}_E$ by adding a dummy vertex $v_{\eta}$ such that the transition costs $w_{\eta,0}=0$ and $w_{\eta,j}=\infty, \forall j\in \{1,\ldots\eta-1\}$. Then, the expanded weight matrix of size $\left(\eta+1\right) \times \left(\eta+1\right)$ is given as
\begin{equation}       
\mathcal{W}_{E}=\left(                 
  \begin{array}{cccccc}
    0 & w_{01} & w_{02} & \cdots & w_{0(\eta-1)} & 0\\
    w_{10} & 0 & w_{12} & \cdots & w_{1(\eta-1)} & \infty\\
    w_{20} & w_{21} & 0 & \cdots & w_{2(\eta-1)} & \infty\\
    \vdots & \vdots & \vdots & \vdots & \vdots & \vdots\\
    w_{(\eta-1)0} & \cdots &  & \cdots & 0 & \infty\\
    0 & \infty & \infty & \cdots & \infty & 0\\
  \end{array}
\right)                
\end{equation}

The set of all solutions of the TSP associated with the weight matrix $\mathcal{W}_E$, is denoted as $\mathcal{C}_{\mathcal{W}_E}$, which is the set of all Hamiltonian cycles, i.e., the set of all cyclic permutations of the set $\mathcal{V}_E$. A Hamiltonian cycle $\mathcal{C}\in\mathcal{C}_{\mathcal{W}_E}$, providing the order in which the vertices are visited, is of the form:
\begin{equation*}\label{eq:ogm_recursive}
\mathcal{C}=\big(v(\lambda) \in {\mathcal{V}_E}, \lambda=0,\ldots\eta+1:  v(\eta+1)=v(0)= v_{\eta}\big)
\end{equation*} 
where $v(\lambda)$ is the vertex visited at step $\lambda$, and $v(\lambda)\neq v(\lambda')$, $\forall \lambda\neq\lambda'; \lambda,\lambda'=0,1,\ldots \eta$. The cost of a cycle $\mathcal{C}$ is given as:
\begin{equation}
\mathcal{J}\left(\mathcal{C}\right)= \sum_{\lambda=0}^{\eta}w_{v(\lambda)v(\lambda+1)}
\end{equation}
Then, the optimal Hamiltonian cycle $\mathcal{C}^{\star}$ is
\begin{equation}
\mathcal{C}^{\star} \in \mathop{\argmin}_{\mathcal{C} \in \mathcal{C}_{\mathcal{W}_E}}\mathcal{J}\left(\mathcal{C}\right)
\end{equation}  

Since the TSP problem is NP hard~\cite{aarts2003}, we utilize a heuristic approach to obtain a feasible solution. First, we adopt the nearest neighbor algorithm~\cite{aarts2003} to obtain an initial tour which starts and ends at the dummy vertex. Since $w_{\eta,0}=0$, $v(1)=v_0$. Then, the 2-opt algorithm~\cite{aarts2003} is applied over this initial tour for further improvement. The 2-opt algorithm iteratively removes two non-adjacent edges and replaces them with two different edges to minimize the length until no improvement can be achieved, thus achieving the optimized tour. Then, the optimized node sequence can be obtained by removing the dummy vertex from the Hamiltonian cycle $\mathcal{C}^{\star}$.

\begin{figure*}[t]
    \captionsetup[subfloat]{farskip=2pt,captionskip=1pt}
    \hspace{\fill}
    \subfloat[Five randomly generated scenes for performance validation.]
    {\includegraphics[width=1\textwidth]{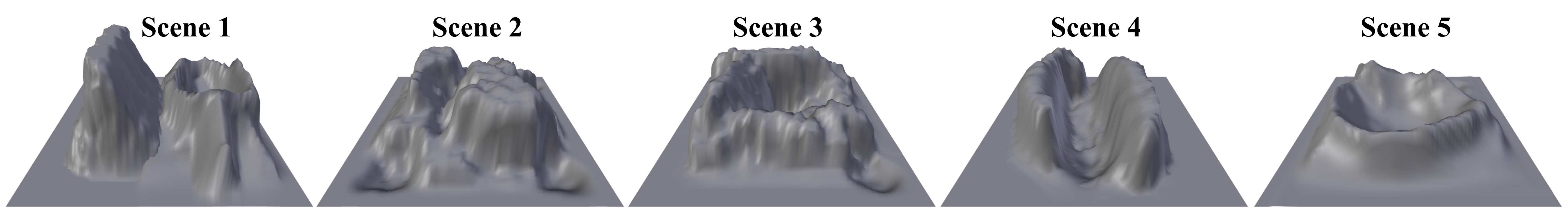}\label{fig:scenario}}
     \hspace{\fill}
    \subfloat[Coverage trajectory using CT-CPP in scene 1. Node traversal sequence: $n_1^0 \rightarrow n_1^1 \rightarrow n_1^2 \rightarrow n_2^2 \rightarrow n_2^3 \rightarrow n_3^2 \rightarrow n_4^2 \rightarrow n_5^2 \rightarrow n_6^2 \rightarrow n_7^2 \rightarrow n_1^3$.]{
        \includegraphics[width=0.5\textwidth]{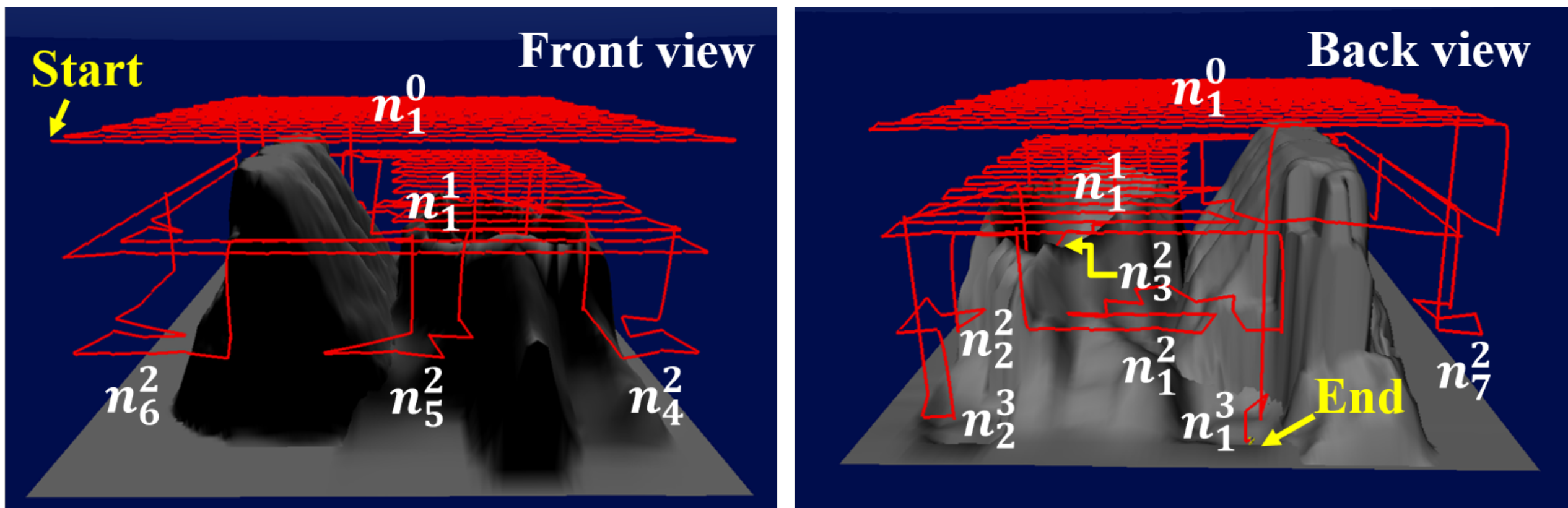}\label{fig:Trajectory-CT-TSP-CPP}}
  \hspace{\fill}
    \subfloat[Coverage trajectory using TF-CPP in scene 1.]{
         \includegraphics[width=0.5\textwidth]{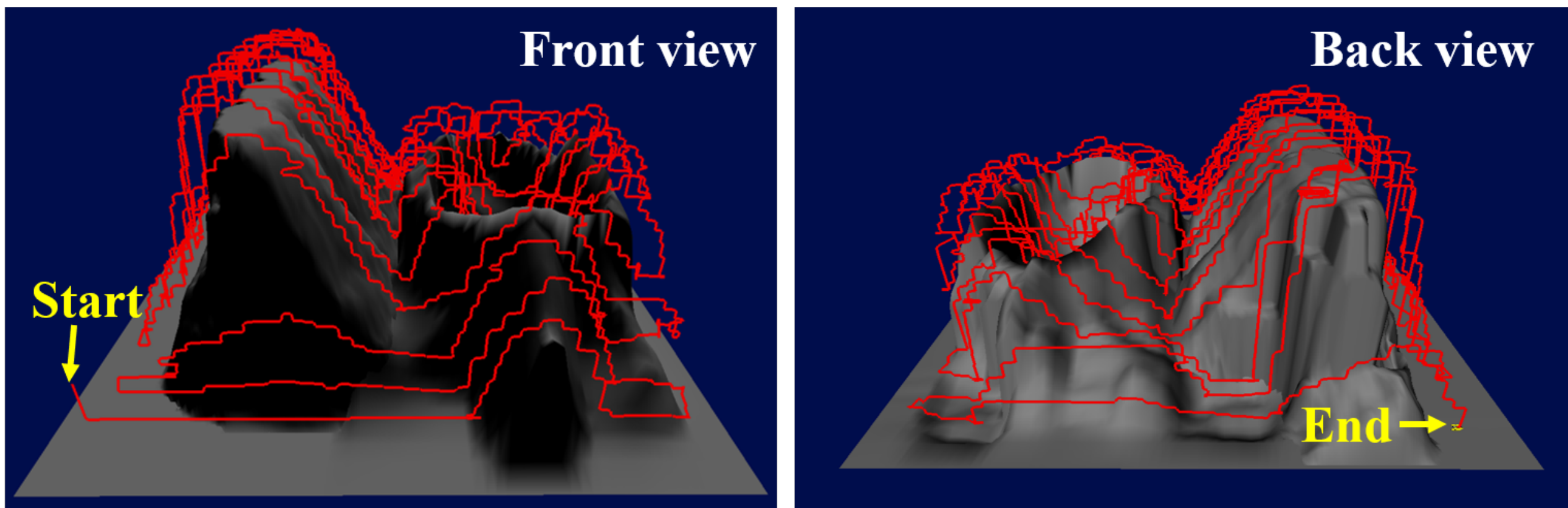}\label{fig:Trajectory-TF-CPP}}
     \hspace{\fill}
     
    \subfloat[Reconstructed map with point set obtained using CT-CPP in scene 1.]{
        \includegraphics[width=0.49\textwidth]{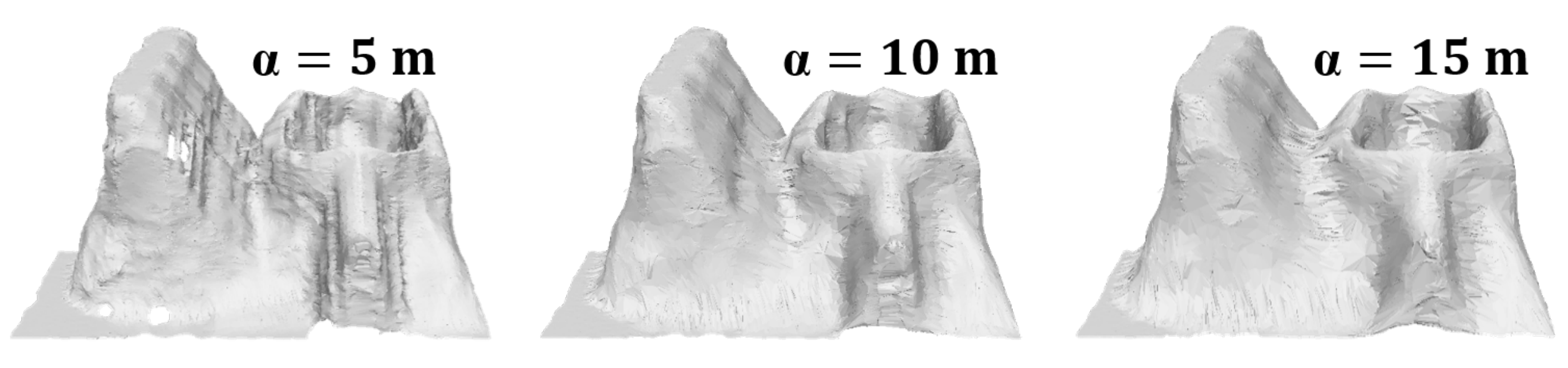}\label{fig:Reconstruction-CT-TSP-CPP}}
         \hspace{\fill}
    \subfloat[Reconstructed map with point set obtained using TF-CPP in scene 1.]{
         \includegraphics[width=0.49\textwidth]{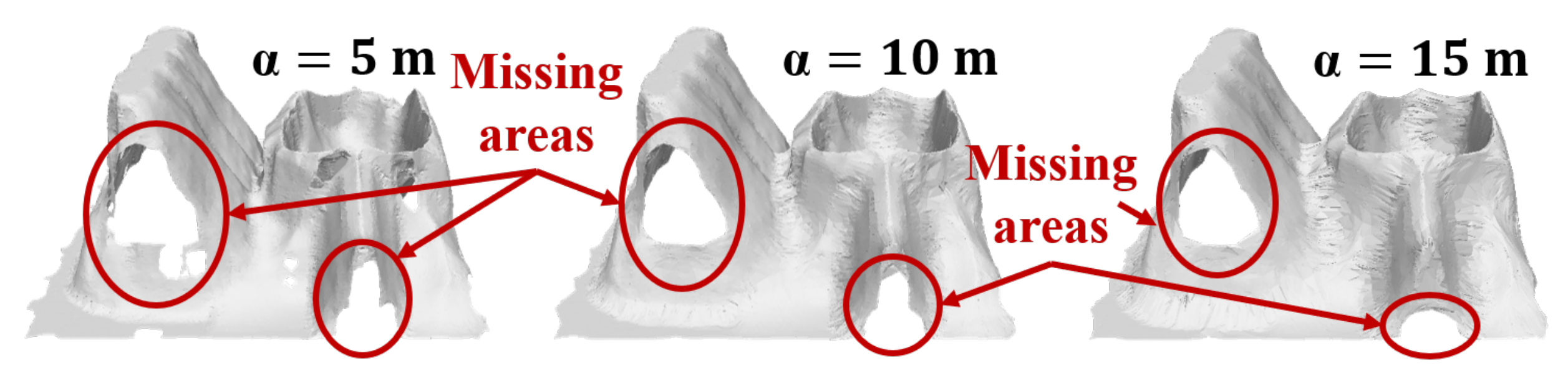}\label{fig:Reconstruction-TF-CPP}}
          \hspace{\fill}

    \subfloat[Comparison of trajectory lengths.]{
        \includegraphics[width=0.32\textwidth]{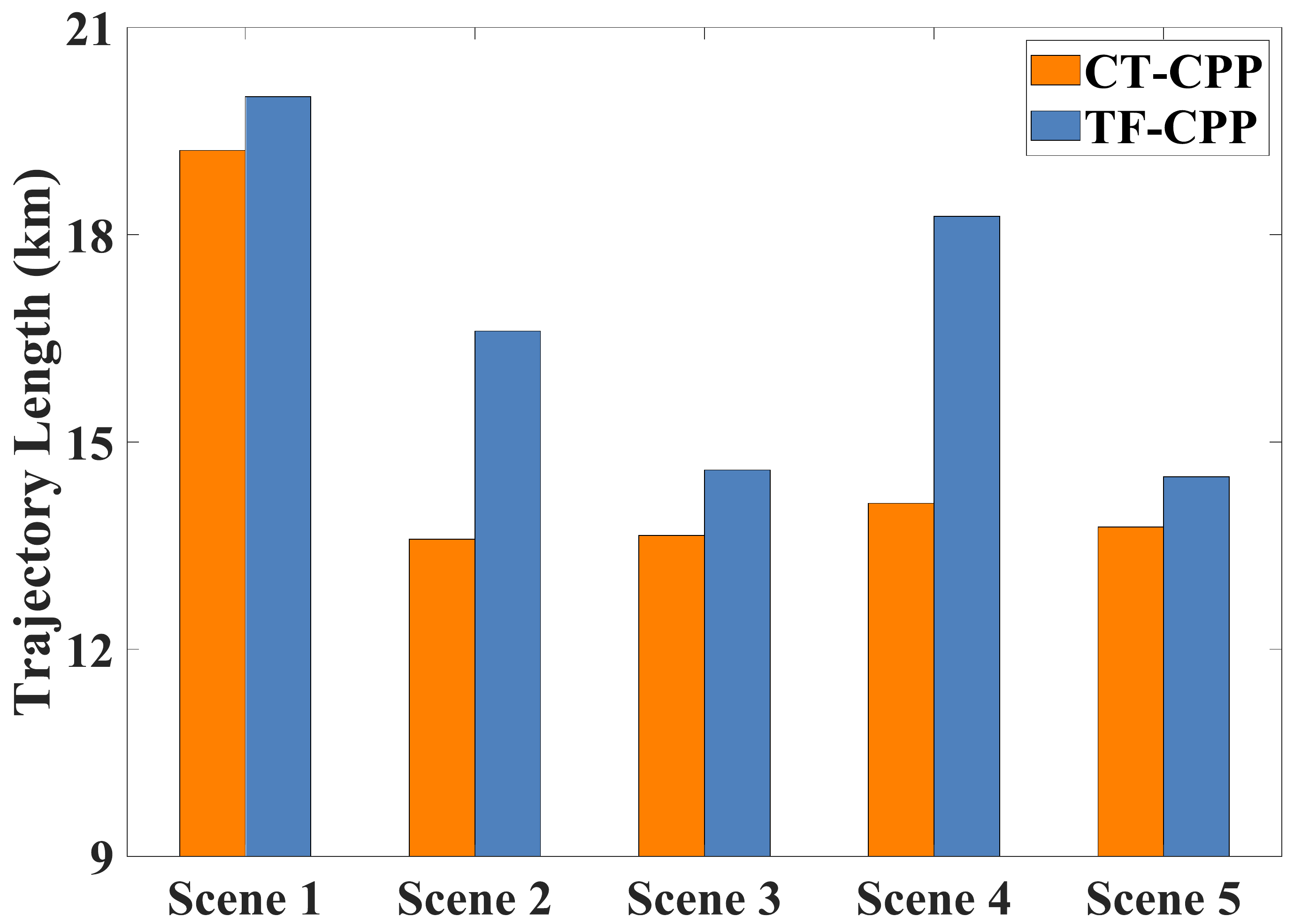}\label{fig:trajectorylength}}
 \hspace{\fill}
    \subfloat[Comparison of energy consumption.]{
         \includegraphics[width=0.32\textwidth]{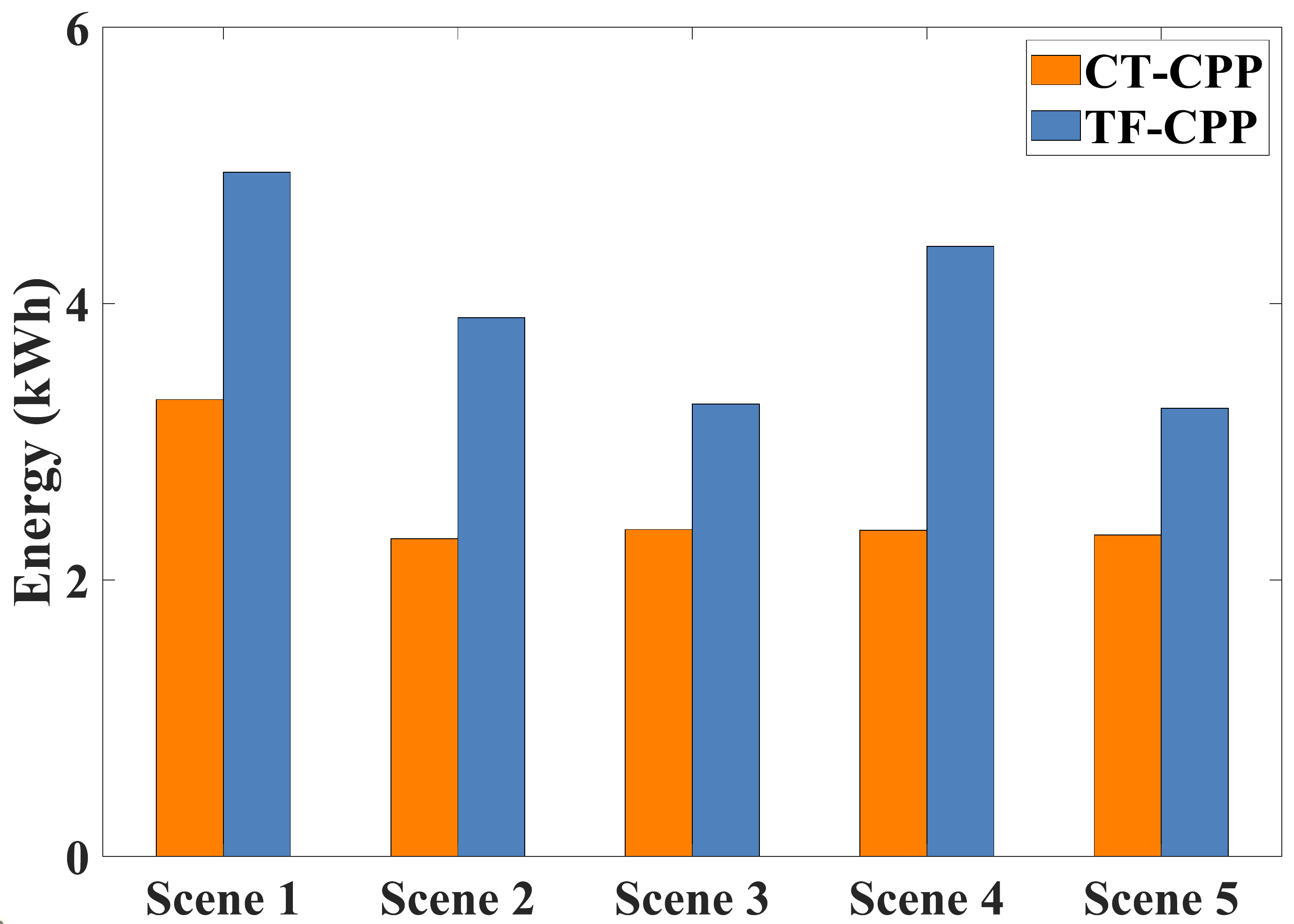}\label{fig:energyCompare}}
    \hspace{\fill}
    \subfloat[Surface reconstructed error using various $\alpha$.]{
        \includegraphics[width=0.32\textwidth]{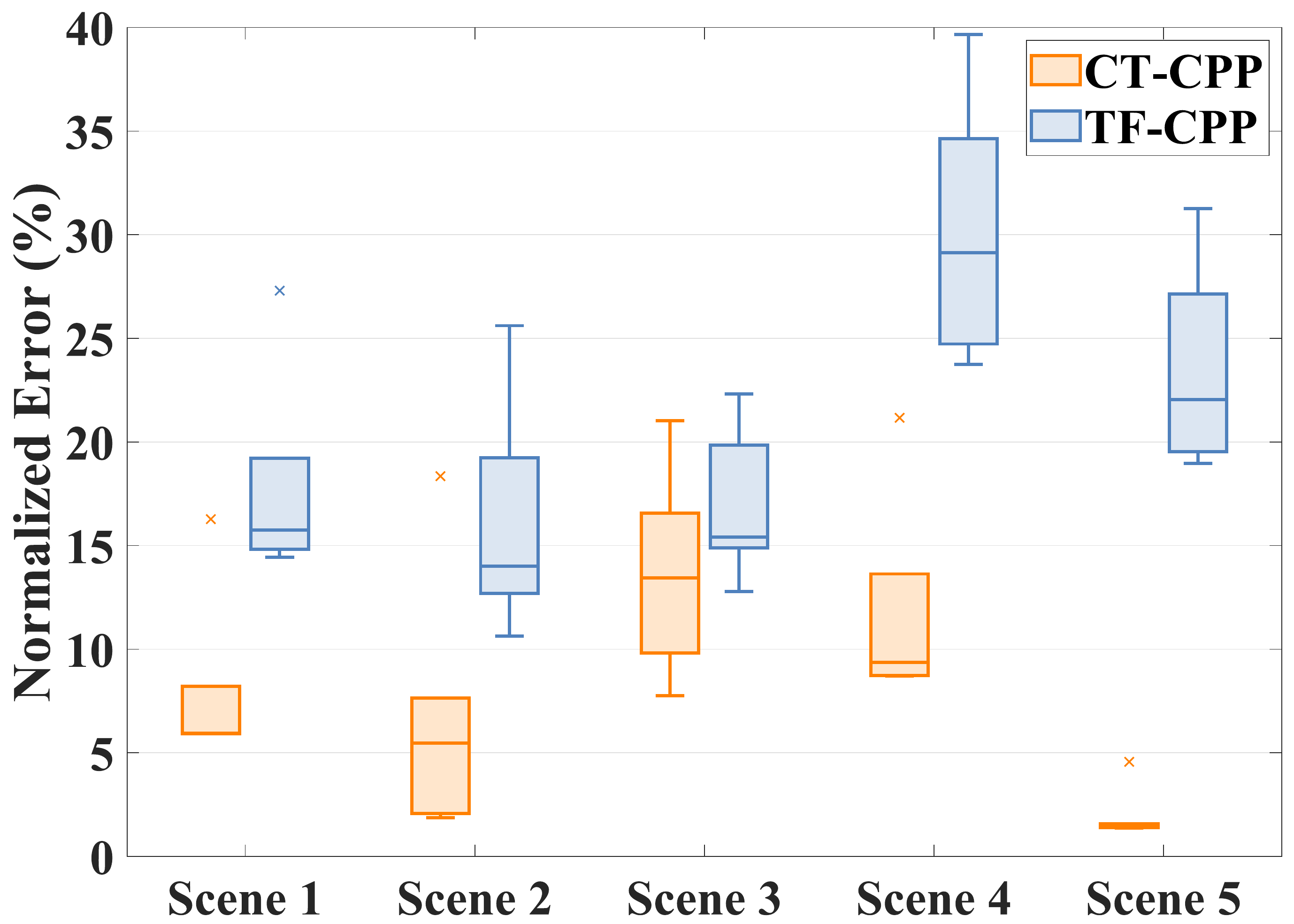}\label{fig:reconstructed_error}}
           \hspace{\fill}
           
    \caption{Performance evaluation for CT-CPP as compared to TF-CPP.}\label{fig:Result}
\vspace{-15pt}
\end{figure*}

Fig.~\ref{fig:TSP} shows an example of the solution procedure where $\mathcal{V}=\left\{v_{0},v_{1},v_{2},v_{3}\right\}$, with $v_{0}$ as the start vertex. A dummy vertex $v_{4}$ is inserted into the weighted graph. The optimized tour generated by the TSP solver is shown below. The final node sequence is obtained by removing the dummy vertex $v_{4}$. Algorithm 1 summarizes the  CT-CPP method. Once the AUV covers a $Target$ (\textbf{Line 5}), it is marked as explored (\textbf{Line 7}), removed from $\mathcal{N}_{U}$ (\textbf{Line 8}) and added to $\mathcal{N}_{E}$ (\textbf{Line 9}). Then, the children of $Target$ are identified (\textbf{Line 10}), added to $N_{U}$ (\textbf{Line 11}), and the CT is updated (\textbf{Line 12}). Finally, a new $Target$ is assigned using the $TSP$ optimizer (\textbf{Line 13}). The algorithm stops when no unexplored nodes are available.

\begin{algorithm}[t] 
\caption{Coverage Tree based 3D CPP Algorithm.} 
\label{alg:algorithm} 
\begin{algorithmic}[1] 
\State $\mathcal{Q}$ $\leftarrow$ $n^0_1$;  // initialize the tree with the root node 
\State $\mathcal{N}_{new}$ $\leftarrow$ $\emptyset$; $\mathcal{N}_{E}$ $\leftarrow$ $\emptyset$; $\mathcal{N}_{U}$ $\leftarrow$ $n^0_1$;  
\State $Target$ $\leftarrow$ $n^0_1$; // set target to be the root node
\While{$\mathcal{N}_{U}$ $\neq$ $\emptyset$}
\State \textbf{Cover}($Target$);
\If{IsComplete($Target$)}
\State Mark $Target$ as Explored;
\State $\mathcal{N}_{U}$ $\leftarrow$ $\mathcal{N}_{U}-Target$;
\State $\mathcal{N}_{E}$ $\leftarrow$ $\mathcal{N}_{E}$$\cup$$Target$;
\State $\mathcal{N}_{new}$ $\leftarrow$ \textbf{Children}($Target$);
\State $\mathcal{N}_{U}$ $\leftarrow$ $\mathcal{N}_{U}$$\cup$ $\mathcal{N}_{new}$;
\State $\mathcal{Q}\leftarrow \mathcal{Q} + \mathcal{N}_{new}$; \ \ where $'+'$ denotes update of the tree with nodes and associated branches;
\State $Target$ $\leftarrow$ \textbf{Assign}($\mathcal{N}_{U}$,$AUV_{pos}$);
\EndIf
\EndWhile
\State\Return $\mathcal{Q}$; \vspace{-3pt}
\end{algorithmic}
\end{algorithm}
\setlength{\textfloatsep}{0pt}

\vspace{-22pt}
\subsection{Computational Complexity Analysis}\label{algorithm_analysis}\vspace{-6pt}
The CT is built incrementally, where upon covering a node the AUV: i) updates the tree by adding the children of this node and ii) computes the next target node. The computational complexity of these processes is as follows. 

Suppose the AUV has covered node $n_s^{\ell}$ corresponding to the subregion $\mathcal{A}^{\ell}_{s}$. During its coverage it generates the POM for the corresponding subregion on the plane below with the same size as $\mathcal{A}^{\ell}_{s}$. Let $\mathcal{T}^{\ell}_{s}$ denote the number of cells in $\mathcal{A}^{\ell}_{s}$. Then, the closing operation with a $3\times 3$ element is applied to this POM (Section~\ref{SafetyMap}), which has $O(|\mathcal{T}^{\ell}_{s}|)$ complexity. Next, the updated POM is transformed into a symbolic map which has $O(|\mathcal{T}^{\ell}_{s}|)$ complexity. Thereafter, the \textit{floodfill} algorithm is applied to this symbolic map to find disconnected subregions, which has $O(|\mathcal{T}^{\ell}_{s}|)$ complexity. These subregions are added to the tree as new nodes. Next, a TSP of size $\eta+1$ including the dummy node is formulated (Section~\ref{traversal_strategy}) to optimize for the tree traversal sequence. An initial solution is obtained using a nearest-neighbor based heuristic approach with $O(\left(\eta+1\right)^2)$ complexity~\cite{aarts2003}. Then, the 2-Opt greedy algorithm is applied to improve the solution with $O\left(\eta+1\right)$~\cite{aarts2003} complexity. Since $|\mathcal{T}^{\ell}_{s}|$ is typically larger than $\eta$, the overall complexity for updating the tree and finding the next target node is $O(|\mathcal{T}^{\ell}_{s}|).$ The average computation time for these processes was $\sim 0.078s$ on a 3.40GHz computer with 16GB RAM.

\vspace{-12pt}
\section{Results}\label{sec:results}
\vspace{-0pt}
The proposed CT-CPP method was validated on a high-fidelity underwater robotic simulator called UWSim~\cite{prats2012}, that interfaces with external control algorithms through the Robot Operating System (ROS) to simulate the AUV and its sensors. 

The AUV, Girona 500, of size $1.5\text{m} \times 1\text{m} \times 1\text{m}$, maximum operating depth of $500\text{m}$, and 8 thrusters to control all the degrees of freedom~\cite{ribas2012}, was simulated with a speed $v=1\text{m/s}$. It is equipped with two multi-beam sonar sensors, one facing downward for terrain data collection, and the other facing forward for obstacle avoidance. The aperture angle, sensing range, number of beams and  sampling interval of the sonars were selected as $\theta = 120^{\circ}$, $r = 150\text{m}$, $M=128$, and $ \Delta t=1$\text{s}, respectively~\cite{melvin2015}. A Doppler Velocity Log (DVL) and an Inertial Measurement Unit (IMU) were used to obtain the velocity and heading angle of the AUV at any time. And a Long Baseline acoustic localization (LBL) system~\cite{shen2016} was used to estimate its location in the GPS-denied environment.

Five scenes of size $450\text{m}\times 450\text{m} \times 400\text{m}$ were randomly generated for performance validation, as shown in Fig.~\ref{fig:scenario}. For trajectory visualization purpose,  Fig.~\ref{fig:Trajectory-CT-TSP-CPP} shows scene 1 which was  sliced using a set of 4 planes $\{\mathcal{A}^\ell \subset \mathbb{R}^2, \ell = 0, \ldots 3\}$. Then, on each plane $\mathcal{A}^{\ell}$ a tiling $\mathcal{T}^{\ell}$ was constructed consisting of $18 \times 18$ cells, with each cell of dimension $25\text{m} \times 25\text{m}$. The lap width  was chosen as $w = 25\text{m}$. Subsequently, the distance between two consecutive planes was computed as $\Delta h=85\text{m}$ using Prop. (\ref{eq:depth}). 
Fig.~\ref{fig:Trajectory-CT-TSP-CPP} and Fig.~\ref{fig:Trajectory-TF-CPP} present a comparative evaluation of the coverage trajectories generated by the CT-CPP and the TF-CPP methods~\cite{hert1996}, respectively, in scene 1.

The downward-facing multi-beam sonar sensor constantly collected the terrain data to create a point cloud of $\sim2$ million data points using the proposed method. Upon completion of the coverage process, we used the Computational Geometry Algorithms Library (CGAL)~\cite{alliez2010} to filter the data. We adopted a grid-based filtering approach~\cite{alliez2010} to abstract the point set to reduce the computational costs, and the $k$-nearest neighbors approach~\cite{alliez2010} to remove the distance-based outliers. Thereafter, the $\alpha$-shapes algorithm~\cite{edelsbrunner1994} provided in the MeshLab software~\cite{cignoni2008} was used for terrain reconstruction, where $\alpha$-shape is a standard geometric tool to reconstruct surfaces from an unorganized point cloud. It relies on a parameter $\alpha \in \mathbb{R}^+$ to control the desired level of details. Fig.~\ref{fig:Reconstruction-CT-TSP-CPP} and Fig.~\ref{fig:Reconstruction-TF-CPP} show the reconstructed surfaces for scene 1 using the CT-CPP and TF-CPP methods, respectively. Fig.~\ref{fig:Reconstruction-CT-TSP-CPP} shows that the proposed method was able to fully reconstruct the underwater terrain; while Fig.~\ref{fig:Reconstruction-TF-CPP} shows that the TF-CPP method missed the side surfaces of high mountains, as explained earlier in Section~\ref{sec:review}.

Three metrics are used for performance evaluation as follows: i) \textit{trajectory length}, ii) \textit{energy consumption}, and iii) \textit{surface reconstruction error}. Fig.~\ref{fig:trajectorylength} shows the total trajectory lengths for all five scenes, which indicate that the CT-CPP method took the shorter path lengths. To compare the energy consumption, we used the energy model~\cite{de2014}:
$\Delta E=k_1 \Delta_h + k_2 \Delta_v$, where $k_1 = 557.24 \text{J/m}$, $k_2=1118.13 \text{J/m}$ and $\Delta_h$ and $\Delta_v$ represent the movements on the horizontal and vertical direction, respectively. A nonlinear energy model could be used in future for better energy estimates. Fig.~\ref{fig:energyCompare} shows that the CT-CPP method consumed less energy in all five scenes as compared to the TF-CPP method due to its frequent vertical motions. Further, we numerically evaluated the reconstruction errors using different $\alpha$. Specifically, for each $\alpha$, we first obtained two sets of 3D sample points from the actual surface and the reconstructed surface, respectively. Then, we partitioned the search area $\mathcal{U}$ into cuboids of size $5\text{m} \times 5\text{m} \times 400\text{m}$, and calculated the average height of the sample points inside each cuboid for both point sets. Note that the average height of sample points inside the cuboid occupied by the missing area is equal to zero. Finally, the Root Mean Square Error (RMSE) is computed between the reconstructed surface and the actual surface. Fig.~\ref{fig:reconstructed_error} shows the box plots of normalized RMSE with respect to the ocean depth for different $\alpha$ values using CT-CPP and TF-CPP methods. As seen, CT-CPP performs better in all five scenes. Thus, CT-CPP is an alternate method to TF-CPP and it is expected to outperform the TF-CPP method in a significant number of scenarios; however, it is possible that for relatively planar scenarios TF-CPP might perform better. 

\vspace{-8pt}
\section{Conclusions}\label{sec:conclusions}
The letter presents a 3D CPP method, called CT-CPP, for unknown terrain reconstruction. CT-CPP incrementally builds a CT as the environment is explored, where the tree traversal sequence is optimized  using a related TSP. It is  shown that CT-CPP is computationally efficient and guarantees complete coverage of projectively planar surfaces. The method is comparatively evaluated with an existing method, called TF-CPP, on a high-fidelity underwater simulator. The results show that CT-CPP results in reduced trajectory lengths, energy consumption and reconstruction error.

\vspace{-9pt}
\bibliographystyle{IEEEtran}
\bibliography{reference}

\begin{thebibliography}{10}
\providecommand{\url}[1]{#1}
\csname url@samestyle\endcsname
\providecommand{\newblock}{\relax}
\providecommand{\bibinfo}[2]{#2}
\providecommand{\BIBentrySTDinterwordspacing}{\spaceskip=0pt\relax}
\providecommand{\BIBentryALTinterwordstretchfactor}{4}
\providecommand{\BIBentryALTinterwordspacing}{\spaceskip=\fontdimen2\font plus
\BIBentryALTinterwordstretchfactor\fontdimen3\font minus
  \fontdimen4\font\relax}
\providecommand{\BIBforeignlanguage}[2]{{%
\expandafter\ifx\csname l@#1\endcsname\relax
\typeout{** WARNING: IEEEtran.bst: No hyphenation pattern has been}%
\typeout{** loaded for the language `#1'. Using the pattern for}%
\typeout{** the default language instead.}%
\else
\language=\csname l@#1\endcsname
\fi
#2}}
\providecommand{\BIBdecl}{\relax}
\BIBdecl

\bibitem{galceran2013}
E.~Galceran and M.~Carreras, ``A survey on coverage path planning for
  robotics,'' \emph{Robot. Auton. Syst.}, vol.~61, no.~12, pp. 1258--1276,
  2013.

\bibitem{song2018}
J.~Song and S.~Gupta, ``$\epsilon^\star$: An online coverage path planning
  algorithm,'' \emph{IEEE Trans. Robot.}, vol.~34, pp. 526--533, 2018.

\bibitem{palomeras2018}
N.~Palomeras, N.~Hurt{\'o}s, M.~Carreras, and P.~Ridao, ``Autonomous mapping of
  underwater 3-\uppercase{D} structures: From view planning to execution,''
  \emph{IEEE Robot. Autom. Lett.}, vol.~3, no.~3, pp. 1965--1971, 2018.

\bibitem{shen2019}
Z.~Shen, J.~P. Wilson, and S.~Gupta, ``An online coverage path planning
  algorithm for curvature-constrained \uppercase{AUV}s,'' in \emph{Proc.
  OCEANS'19 MTS/IEEE}, SEATTLE, WA, USA, Oct. 2019, pp. 1--5.

\bibitem{vempati2018}
A.~S. Vempati, M.~Kamel, N.~Stilinovic, Q.~Zhang, D.~Reusser, I.~Sa, J.~Nieto,
  R.~Siegwart, and P.~Beardsley, ``Paintcopter: An autonomous \uppercase{UAV}
  for spray painting on three-dimensional surfaces,'' \emph{IEEE Robot. Autom.
  Lett.}, vol.~3, no.~4, pp. 2862--2869, 2018.

\bibitem{atkar2005}
P.~N. Atkar, A.~Greenfield, D.~C. Conner, H.~Choset, and A.~A. Rizzi, ``Uniform
  coverage of automotive surface patches,'' \emph{Int. J. Robot. Res.},
  vol.~24, no.~11, pp. 883--898, 2005.

\bibitem{vidal2017}
E.~Vidal, J.~D. Hern{\'a}ndez, K.~Istenic, and M.~Carreras, ``Online view
  planning for inspecting unexplored underwater structures.'' \emph{IEEE Robot.
  Autom. Lett.}, vol.~2, no.~3, pp. 1436--1443, 2017.

\bibitem{bircher2018receding}
A.~Bircher, M.~Kamel, K.~Alexis, H.~Oleynikova, and R.~Siegwart, ``Receding
  horizon path planning for 3\uppercase{D} exploration and surface
  inspection,'' \emph{Auton. Robots}, vol.~42, no.~2, pp. 291--306, 2018.

\bibitem{song2020online}
S.~Song, D.~Kim, and S.~Jo, ``Online coverage and inspection planning for
  3\uppercase{D} modeling,'' \emph{Auton. Robots}, vol.~44, no.~8, pp.
  1431--1450, 2020.

\bibitem{mukherjee2011}
K.~Mukherjee, S.~Gupta, A.~Ray, and S.~Phoha, ``Symbolic analysis of sonar data
  for underwater target detection,'' \emph{IEEE J. Ocean. Eng.}, vol.~36,
  no.~2, pp. 219--230, 2011.

\bibitem{SGH13}
J.~Song, S.~Gupta, J.~Hare, and S.~Zhou, ``Adaptive cleaning of oil spills by
  autonomous vehicles under partial information,'' in \emph{Proc. OCEANS'13
  MTS/IEEE}, San Diego, CA, USA, Sep. 2013, pp. 1--5.

\bibitem{jin2011}
J.~Jin and L.~Tang, ``Coverage path planning on three-dimensional terrain for
  arable farming,'' \emph{J. Field Robot.}, vol.~28, no.~3, pp. 424--440, 2011.

\bibitem{acar2002}
E.~U. Acar and H.~Choset, ``Sensor-based coverage of unknown environments:
  Incremental construction of morse decompositions,'' \emph{Int. J. Robot.
  Res.}, vol.~21, no.~4, pp. 345--366, 2002.

\bibitem{sadat2014}
S.~A. Sadat, J.~Wawerla, and R.~T. Vaughan, ``Recursive non-uniform coverage of
  unknown terrains for \uppercase{UAV}s,'' in \emph{Proc. IEEE Int. Conf.
  Intell. Robots Syst.}, 2014, pp. 1742--1747.

\bibitem{thrun2005}
S.~Thrun, W.~Burgard, and D.~Fox, \emph{Probabilistic robotics}.\hskip 1em plus
  0.5em minus 0.4em\relax MIT press, 2005.

\bibitem{soille2013}
P.~Soille, \emph{Morphological image analysis: principles and
  applications}.\hskip 1em plus 0.5em minus 0.4em\relax Springer Science \&
  Business Media, 2013.

\bibitem{prats2012}
M.~Prats, J.~P{\'e}rez, J.~J. Fern{\'a}ndez, and P.~J. Sanz, ``An open source
  tool for simulation and supervision of underwater intervention missions,'' in
  \emph{Proc. IEEE Int. Conf. Intell. Robots Syst.}, 2012, pp. 2577--2582.

\bibitem{shen2017}
Z.~Shen, J.~Song, K.~Mittal, and S.~Gupta, ``Autonomous 3-\uppercase{D} mapping
  and safe-path planning for underwater terrain reconstruction using
  multi-level coverage trees,'' in \emph{Proc. OCEANS'17 MTS/IEEE}, Anchorage,
  AK, USA, Sep. 2017, pp. 1--6.

\bibitem{galceran2015}
E.~Galceran, R.~Campos, N.~Palomeras, D.~Ribas, M.~Carreras, and P.~Ridao,
  ``Coverage path planning with real-time replanning and surface reconstruction
  for inspection of three-dimensional underwater structures using autonomous
  underwater vehicles,'' \emph{J. Field Robot.}, vol.~32, no.~7, pp. 952--983,
  2015.

\bibitem{lee2009}
T.-S. Lee, J.-S. Choi, J.-H. Lee, and B.-H. Lee, ``3-\uppercase{D} terrain
  covering and map building algorithm for an \uppercase{AUV},'' in \emph{Proc.
  IEEE Int. Conf. Intell. Robots Syst.}, 2009, pp. 4420--4425.

\bibitem{lumelsky1990}
V.~J. Lumelsky, S.~Mukhopadhyay, and K.~Sun, ``Dynamic path planning in
  sensor-based terrain acquisition,'' \emph{IEEE Trans. Robot. Autom.}, vol.~6,
  no.~4, pp. 462--472, 1990.

\bibitem{hert1996}
S.~Hert, S.~Tiwari, and V.~Lumelsky, ``A terrain-covering algorithm for an
  \uppercase{AUV},'' \emph{J. Auton. Robots}, vol.~3, no.~2, pp. 91--119, 1996.

\bibitem{cheng2008}
P.~Cheng, J.~Keller, and V.~Kumar, ``Time-optimal \uppercase{UAV} trajectory
  planning for 3\uppercase{D} urban structure coverage,'' in \emph{Proc. IEEE
  Int. Conf. Intell. Robots Syst.}, 2008, pp. 2750--2757.

\bibitem{sadat2015}
S.~A. Sadat, J.~Wawerla, and R.~Vaughan, ``Fractal trajectories for online
  non-uniform aerial coverage,'' in \emph{Proc. IEEE Int. Conf. Robot. Autom.},
  2015, pp. 2971--2976.

\bibitem{edelsbrunner1994}
H.~Edelsbrunner and E.~P. M{\"u}cke, ``Three-dimensional alpha shapes,''
  \emph{ACM Trans. Graph.}, vol.~13, no.~1, pp. 43--72, 1994.

\bibitem{GRP09}
S.~Gupta, A.~Ray, and S.~Phoha, ``Generalized ising model for dynamic
  adaptation in autonomous systems,'' \emph{Euro. Phys. Lett.}, vol.~87, p.
  10009, 2009.

\bibitem{Heckbert1990}
P.~Heckbert, ``A seed fill algorithm,'' in \emph{Graphics Gems}, A.~Glassner,
  Ed.\hskip 1em plus 0.5em minus 0.4em\relax Boston: Academic Press, January
  1990, pp. 275--277.

\bibitem{aarts2003}
E.~Aarts, E.~H. Aarts, and J.~K. Lenstra, \emph{Local search in combinatorial
  optimization}.\hskip 1em plus 0.5em minus 0.4em\relax Princeton University
  Press, 2003.

\bibitem{ribas2012}
D.~Ribas, N.~Palomeras, P.~Ridao, M.~Carreras, and A.~Mallios, ``Girona 500
  \uppercase{AUV}: From survey to intervention,'' \emph{IEEE/ASME Trans.
  Mechatronics}, vol.~17, no.~1, pp. 46--53, 2012.

\bibitem{melvin2015}
G.~D. Melvin and N.~A. Cochrane, ``Multibeam acoustic detection of fish and
  water column targets at high-flow sites,'' \emph{Estuaries and coasts},
  vol.~38, no.~1, pp. 227--240, 2015.

\bibitem{shen2016}
Z.~Shen, J.~Song, K.~Mittal, and S.~Gupta, ``An autonomous integrated system
  for 3-\uppercase{D} underwater terrain map reconstruction,'' in \emph{Proc.
  OCEANS'16 MTS/IEEE}, Monterey, CA, USA, Sep. 2016, pp. 1--6.

\bibitem{alliez2010}
P.~Alliez, L.~Saboret, and N.~Salman, ``Point set processing,'' \emph{CGAL User
  and Reference Manual}, vol.~3, 2010.

\bibitem{cignoni2008}
P.~Cignoni, M.~Callieri, M.~Corsini, M.~Dellepiane, F.~Ganovelli, and
  G.~Ranzuglia, ``Meshlab: an open-source mesh processing tool.'' in
  \emph{Proc. Conf. Eurographics Italian chapter}, 2008, pp. 129--136.

\bibitem{de2014}
V.~De~Carolis, D.~M. Lane, and K.~E. Brown, ``Low-cost energy measurement and
  estimation for autonomous underwater vehicles,'' in \emph{Proc. OCEANS'14
  MTS/IEEE}, Taipei, Taiwan, 2014, pp. 1--5.

\end{thebibliography}
\end{document}